%% file: main.tex
\pdfoutput=1

\documentclass[11pt]{article}

\usepackage[final]{acl}

\usepackage{times}
\usepackage{latexsym}

\usepackage[T1]{fontenc}

\usepackage[utf8]{inputenc}

\usepackage{microtype}

\usepackage{inconsolata}

\usepackage{graphicx}
\usepackage{multirow}
\usepackage{booktabs}
\usepackage{amssymb}
\usepackage{amsmath}
\usepackage{amsthm}
\newtheorem{theorem}{Theorem}[section]  
\usepackage{algorithm}      
\usepackage{algpseudocode}  
\usepackage{hyperref}
\usepackage{enumitem,amssymb}
\newlist{todolist}{itemize}{2}
\setlist[todolist]{label=$\square$}
\usepackage{pifont}
\newcommand{\ie}{\emph{i.e., }}
\newcommand{\eg}{\emph{e.g., }}

\newcommand{\sj}[1]{{\color{black}#1}}

\definecolor{-}{rgb}{0.25,0.41,0.88}
\definecolor{+}{rgb}{0.70,0.13,0.13}

\usepackage[symbol]{footmisc}

\title{Robust Preference Optimization via Dynamic Target Margins}

\author{
    Jie Sun$^1$\hspace{0.5mm},
    Junkang Wu$^1$\hspace{0.5mm},
    Jiancan Wu$^2$\hspace{0.5mm},
    Zhibo Zhu$^1$\hspace{0.5mm}, \\
    \textbf{Xingyu Lu$^1$\hspace{0.5mm},}
    \textbf{Jun Zhou$^1$\hspace{0.5mm},}
    \textbf{Lintao Ma$^1$\footnotemark[1]\hspace{0.5mm},}
    \textbf{Xiang Wang$^3$\footnotemark[1]\hspace{0.5mm}} \\
    $^1$ Ant Group \quad $^2$ Shanghai Key Laboratory of Data Science \\
    $^3$ National University of Singapore \\
    \{kangji.sj, gavin.zzb, lintao.mlt, sing.lxy\}@antgroup.com \\
    \{jkwu0909, wujcan\}@gmail.com \\
    jun.zhoujun@antfin.com \quad xiangwang@u.nus.edu
}

\begin{document}
\maketitle
\footnotetext[1]{Corresponding authors.} 
\input{chapter/0-abstract}
\input{chapter/1-intro}
\input{chapter/2-preliminary}
\input{chapter/3-method}

\input{chapter/4-experiment}
\input{chapter/5-related_work}
\input{chapter/6-conclusion}
\input{chapter/7-limitation}

\bibliography{main}
\newpage
\appendix
\input{chapter/8-appendix}

\end{document}

%% file: chapter/0-abstract.tex
\begin{abstract}
The alignment of Large Language Models (LLMs) is crucial for ensuring their safety and reliability in practical applications. 
Direct Preference Optimization (DPO) has emerged as an efficient method that directly optimizes models using preference pairs, significantly reducing resource demands.
However, the effectiveness of DPO heavily depends on the data quality, which is frequently compromised by noise.
In this work, we propose $\gamma$-PO, a dynamic target margin preference optimization algorithm that adjust reward margins at the pairwise level.
By introducing instance-specific margin calibration, $\gamma$-PO strategically prioritizes high-confidence pairs (those demonstrating higher reward margins) while suppressing potential noise from ambiguous pairs.
Moreover, $\gamma$-PO is a plug-and-play method, compatible with variants of DPO that rely on reward margin between preference pairs.
Across benchmarks such as AlpacaEval2 and Arena-Hard, $\gamma$-PO achieves an average 4.4\% improvement over other baselines, setting new benchmarks for state-of-the-art performance.
Additionally, $\gamma$-PO requires minimal code changes and has a negligible impact on training efficiency, making it a robust solution for enhancing LLMs alignment.
Our codes are available at \href{https://github.com/sunjie279/gammaPO}{https://github.com/sunjie279/gammaPO}.

\end{abstract}

%% file: chapter/1-intro.tex
\section{Introduction}

\begin{figure}[t]
  \includegraphics[width=\columnwidth]{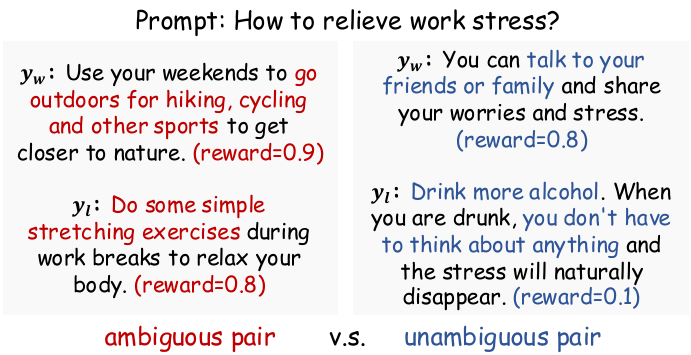}
  \caption{Comparison of ambiguous and unambiguous sample pairs. Ambiguous pairs exhibit narrow reward margins, indicating low confidence in model predictions, whereas unambiguous pairs demonstrate wide reward margins, reflecting high prediction confidence.}
  \label{fig:motivation}
\end{figure}

The alignment of Large Language Models (LLMs) with human values remains critical for their safe deployment \cite{dong2024rlhfworkflowrewardmodeling}.
While Reinforcement Learning from Human Feedback (RLHF) \cite{ouyang2022RLHF} pioneered model alignment through reward modeling and Proximal Policy Optimization (PPO) \cite{schulman2017PPO}, its computational intensity and training instability \cite{zheng2023secretsrlhflargelanguage} motivate simpler alternatives.
Direct Preference Optimization (DPO) \cite{rafailov2024DPO} offers an efficient alternative by directly optimizing preference probabilities with human preference pairs, bypassing the need for explicit reward modeling.

\begin{figure}[t]
  \includegraphics[width=\columnwidth]{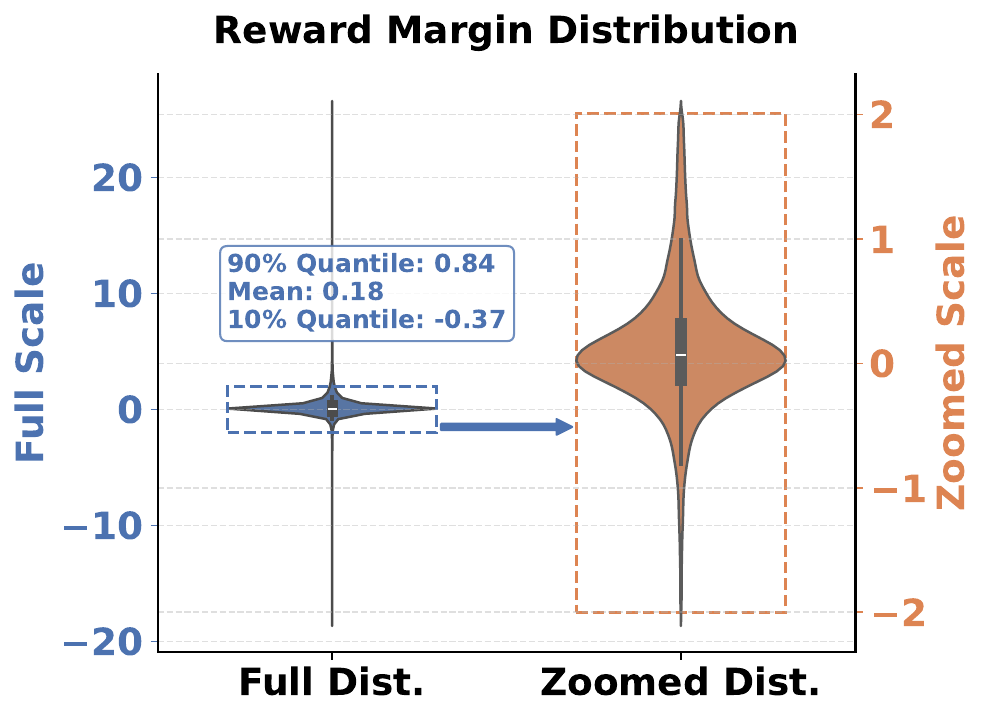}
  \caption{Distribution of the reward margin for the Mistral model using the SimPO objective on the Ultrafeedback Binarized dataset. The \textbf{\textcolor[HTML]{4C72B0}{blue}} violin plot represents the full range of the original distribution, while the \textbf{\textcolor[HTML]{DD8452}{brown}} plot provides a zoomed-in view of the distribution, focusing on the central values close to zero.}
  \label{fig:distribution}
\end{figure}

However, DPO exhibits significant performance \sj{degradation when dealing with data that presents uncertain or ambiguous reward signals (see Figure \ref{fig:motivation}, left panel).
In such cases, narrow reward margins reflect inherent ambiguity in response-reward comparisons, which can lead models to overfit to weak or conflicting preference signals rather than capturing consistent human preferences. 
This challenge is particularly critical as our empirical analysis in Figure \ref{fig:distribution} reveals that a substantial majority of training pairs are concentrated around zero reward margin, necessitating more attention and methodological improvements to effectively handle such ambiguous instances.}

Existing robust alignment approaches typically adopt two distinct paradigms:
(1) applying soft labels through label smoothing techniques to reduce overconfidence in potentially noisy annotations~\cite{chowdhury2024rDPO}, and
(2) implementing rule-based data filtering mechanisms to eliminate ambiguous preference pairs~\cite{wu2024beta_DPO}.
Both categories face critical limitations:
\begin{itemize}
  \item \textbf{One-Size-Fits-All Objective:}
  Current methods \cite{chowdhury2024rDPO,meng2024simpo,amini2024DPOwithoffset,rafailov2024DPO} employ \sj{uniform} loss functions that treat all preference pairs equally, disregarding varying confidence scores across samples.
  
  \item \textbf{Strong Theoretical Assumptions:}
  Prior work requires either known label flip rates~\cite{chowdhury2024rDPO} or specific reward margin distributions \cite{wu2024beta_DPO}, coupled with complex auxiliary mechanisms such as $\beta$-guided filtering that significantly limit practical adoption.
\end{itemize}

\sj{For modeling preference pair ambiguity,} the Bradley-Terry model~\cite{bradley1952BTmodel} provides a probabilistic foundation through $P(y_w \succ y_l) = \sigma(r_w - r_l)$, where smaller reward margin indicates higher uncertainty. 
Building on this insight, recent work SimPO~\cite{meng2024simpo} introduces a fixed target margin $\gamma_0$ to reshape the reward landscape via $\sigma(r_w - r_l - \gamma_0)$. 
While effective in clean data scenarios, this rigid margin fails to adapt to the spectrum of uncertainty inherent in real-world preference pairs. This limitation motivates our key research question:
\begin{center}
\textit{Can adaptive target margin adjustment based on reward margin enhance alignment robustness?}
\end{center}

We propose $\gamma$-PO, a plug-and-play method that dynamically adjusts target margins $\gamma_i$ for each preference pair. Our approach employs confidence-aware margin scaling: for unambiguous pairs with a significant difference in rewards, we increase the target margins to strengthen gradient signals during training. Conversely, for ambiguous pairs with small reward margins, we reduce the target margins to mitigate the influence of uncertain label. This dual mechanism provides two complementary benefits:
(1) Algorithm-agnostic design compatible with any reward margin-based method (\eg DPO, SimPO),
and (2) Theoretical connections to dynamic label smoothing \cite{su2022dsladynamicsmoothlabel}. 
\sj{We derive that $\gamma$-PO inherently implements \textit{dynamic label smoothing} through its dynamic target margins. 
Specifically, smaller reward margins induce larger $\varepsilon$ (\ie smoother labels), effectively smoothing labels for ambiguous pairs. Conversely, larger margins reduce $\varepsilon$ (\ie sharper labels), sharpening supervision for unambiguous samples.}

Experiments across four base models (LLaMA-3-8B-Instruct, Mistral-7B-Instruct, Gemma-2-9B-Instruct, and Qwen-2.5-7B-Instruct) show $\gamma$-PO achieves 4.4\% average gains on AlpacaEval2~\cite{li2023alpacaeval}, Arena-Hard~\cite{arenahard2024_ArenaHard}, and MT-Bench~\cite{mt_bench}, with minimal computational overhead.
Additionally, our plug-and-play approach requires only a few lines of code and has a negligible impact on training efficiency, yet significantly enhances performance.


%% file: chapter/2-preliminary.tex
\section{Preliminaries}

\subsection{Reinforcement Learning from Human Feedback}
While LLMs are pre-trained on extensive datasets, those undergoing only supervised fine-tuning may output harmful or helpless content.
To better align LLMs with human preferences, a widely adopted method is Reinforcement Learning from Human Feedback~\cite{ouyang2022RLHF}, which 
involves three main steps: supervised fine-tuning \cite{zhou2024lima-sft, ding2023enhancing-sft}, reward model training \cite{chen2024odin-reward, gao2023scaling-reward}, and policy model optimization. 
Specifically, the policy model is trained using Proximal Policy Optimization (PPO) \cite{anthony2017thinking-ppo} that maximizes response rewards with Kullback-Leibler (KL) regularization~\cite{kullback1951KLinformation} between the policy model $\pi_\theta$ and a reference model $\pi_\text{ref}$:
\begin{equation}
\begin{aligned}
    &\max_{\pi_\theta}\mathbb{E}_{x\sim\mathcal{D},y\sim\pi_\theta(y|x)} \\
    &\left[r_\phi(x,y)-\beta\mathbb{D}_\text{KL}[\pi_\theta(y|x)||\pi_\text{ref}(y|x)]\right],
\end{aligned}
\end{equation}
where response $y$ is generated from policy $\pi_\theta$ based on the prompt $x$ in dataset $\mathcal{D}$, and $r_\phi$ is the reward model. 
To estimate the labeling uncertainty of tuple $(x,y_w,y_l)$, Bradley-Terry (BT)
model~\cite{bradley1952BTmodel} defines the success probability:
\begin{equation}
\begin{aligned}
    \mathbb{P}(y_w \succ y_l \mid x) 
    &= \sigma (r_\phi(x, y_w) - r_\phi(x, y_l)),
\end{aligned}
\end{equation}
where $\sigma$ is the sigmoid function.

\subsection{Direct Preference Optimization}
DPO \cite{rafailov2024DPO} is a prominent alignment method that directly optimizes the policy model, eliminating the need for training an explicit reward model and significantly reducing the resources required for alignment:
\begin{equation}
\begin{aligned}
    &\mathcal{L}_{\text{DPO}}(\pi_\theta;\pi_{\text{ref}}) = -\mathbb{E}_{(x,y_w,y_l) \sim\mathcal{D}} \\ 
    &\left[ \log \sigma\left( \beta \log \frac{\pi_\theta(y_w|x)}{\pi_{\text{ref}}(y_w|x)} - \beta \log \frac{\pi_\theta(y_l|x)}{\pi_{\text{ref}}(y_l|x)} \right) \right],
\end{aligned}
\end{equation}
where $y_w$ is the prefered response over $y_l$, and $\beta$ is the hyperparameter.

\subsubsection{Preference Optimization with Margin}
\label{sec:marginPO}
Recent studies such as ODPO \cite{amini2024DPOwithoffset} and SimPO \cite{meng2024simpo} propose to modulate the reward difference between preferred and less preferred samples by introducing a target margin $\gamma_0$. We express the generalized preference optimization with a target margin through a unified loss function:
\begin{equation}
\label{marginPO}
\mathcal{L}_\text{margin PO} = -\mathbb{E}_{(x,y_w,y_l)\sim\mathcal{D}} \left[\log \sigma\left( r_w - r_l - \gamma_0 \right)\right].
\end{equation}
The definition of $r_w$ and $r_l$ in each method are summarized in Table \ref{tab:marginPO}.

\subsubsection{Robust Preference Optimization}
To address the challenge of noisy preference data, Robust DPO (rDPO) \cite{chowdhury2024rDPO} has been proposed. rDPO enhances robustness by explicitly incorporating the label flip rate into its formulation, thereby reducing sensitivity to erroneous preference flips. 
The loss function for rDPO is defined as follows:
\begin{equation}
\begin{aligned}
    \mathcal{L}_\text{rDPO}=(1-\varepsilon)\mathcal{L}_\text{DPO}(y_w\succ y_l) + \varepsilon \mathcal{L}_\text{DPO}(y_l\succ y_w), \\
\end{aligned}
\end{equation}
where $\varepsilon$ represents the flip rate of labels. This formulation ensures that when there is no label flip ($\varepsilon = 0$), the rDPO loss reduces to the standard DPO loss. 
rDPO shares similarities with label smoothing techniques, which introduce regularization to soften predictions' confidence. 
In rDPO, the parameter $\varepsilon$ can be seen as a smoothing factor that controls the degree of confidence for the predictions. This provides a mechanism analogous to label smoothing, specifically tailored for handling noisy preferences optimization.

\input{table/marginPO}

%% file: table/marginPO.tex
\begin{table}
  \centering
  \begin{tabular}{lcc}
    \toprule
    \textbf{Methods} & \textbf{$r_w$} & \textbf{$r_l$} \\
    \midrule
    ODPO & $\beta \log \frac{\pi_\theta(y_w|x)}{\pi_{\text{ref}}(y_w|x)}$ & $\beta \log \frac{\pi_\theta(y_l|x)}{\pi_{\text{ref}}(y_l|x)}$ \\
    SimPO & $\frac{\beta}{|y_w|}\log\pi_\theta(y_w|x)$ & $\frac{\beta}{|y_l|}\log\pi_\theta(y_l|x)$ \\
    \bottomrule
  \end{tabular}
    \caption{Definitions of $r_w$ and $r_l$ for alignment method ODPO and SimPO.}
  \label{tab:marginPO}
\end{table}

%% file: chapter/3-method.tex
\begin{figure}[t]
  \includegraphics[width=\columnwidth]{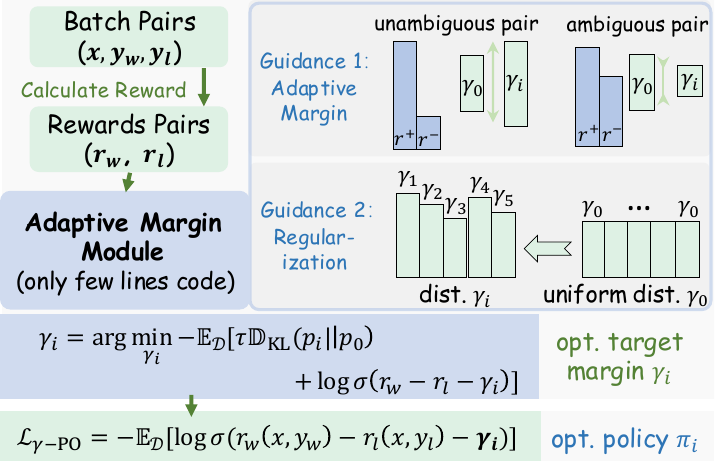}
  \caption{The Dynamic Target Margin Module dynamically adjusts the adaptive target margin ($\gamma_i$) through reward-driven optimization, guided by the dual mechanisms described in Section \ref{sec:guidances}. The optimized $\gamma_i$ subsequently replaces the static margin ($\gamma_0$) in the policy optimization loss function, enabling adaptive margin adjustment throughout the training process.}
  \label{fig:framework}
\end{figure}

\section{Method}
In this section, we first present the methodology for deriving the adaptive target margin. Next, we present our proposed method, $\gamma$-PO, along with its two variants, $\gamma$-DPO and $\gamma$-SimPO. Finally, we conduct a comprehensive comparative analysis with rDPO~\cite{chowdhury2024rDPO}, highlighting its superior robustness and alignment capabilities.

\subsection{Derivation of Adaptive Target Margin} \label{sec:guidances}

The core objective of preference optimization with margin is to enforce $r_w > r_l + \gamma_0$ where $\gamma_0$ is a fixed margin.
Our framework extends this by introducing instance-specific margins to address ambiguous reward signal in preference optimization, ensuring that:
\begin{equation}
r_w - r_l > \gamma_i \Longleftrightarrow r_l - r_w + \gamma_i <0.
\end{equation}
To optimize $\gamma_i$, we focus on instances that violate this condition (where $r_l - r_w + \gamma_i \geq 0$).
This leads to the following optimization objective:
\begin{equation}
\label{equ:loss-gamma1}
    \min_{\gamma_i} \max\left\{0, r_l - r_w + \gamma_i\right\}.
\end{equation}
For computational efficiency, we reformulate this objective using the logarithmic approximation (\ie $\max\{0, y-x\} \approx -\log \frac{\exp(x)}{\exp{(x)}+\exp{(y)}}$), yielding:
\begin{equation}
\begin{aligned}
&\min_{\gamma_i} -\log\frac{\exp(r_w)}{\exp(r_w)+\exp(r_l)\exp(\gamma_i)}\\
    = &\min_{\gamma_i} -\log\sigma(r_w - r_l - \gamma_i).
\end{aligned}
\end{equation}
This logarithmic formulation enhances numerical stability during gradient computation while maintaining the theoretical guarantees.

Our dynamic target margin $\gamma_i$ are guided by two principles:
\textbf{(1) Adaptive Margin: }The preference pairs with larger reward margins represent higher-confidence annotations and should receive stronger learning signals, while ambiguous pairs require more conservative treatment, and 
\textbf{(2) Regularization: } To mitigate the risk of margin collapse (\ie $\gamma_i\rightarrow -\infty$), the adaptive target margin distribution must preserve a reasonable scale relative to a baseline target margin $\gamma_0$. To achieve this, we regularize the target margin distribution by minimizing the KL divergence with respect to a uniform prior. 
This dual-objective formulation is expressed as:
\begin{equation}
\begin{aligned}
    \gamma_i 
    = \arg\min_{\gamma_i} -\mathbb{E}_{\mathcal{D}} [ &\log\sigma(r_w - r_l - \gamma_i) \\
    &- \tau \, \mathbb{D}_\text{KL}(p_i \| p_0) ],
\end{aligned}
\label{equ:loss-gamma2}
\end{equation}
where $p_i=\gamma_i / \sum_j \gamma_j$ and $p_0$ represent the normalized $\gamma_i$ distribution and the uniform distribution, respectively. Regularization coefficient $\tau$ controls the trade-off between adaptation flexibility and stability.  
The complete derivation and computational details of $\gamma_i$ are provided in Appendix \ref{apd:calculate-gamma}. 
For clarity, we summarize the key notations used in this paper in Table \ref{tab:notations}.

\subsection{Generalizability of $\gamma$-PO}

Building upon the dynamic target margin $\gamma_i$ derived in Section \ref{sec:marginPO}, we extend the conventional fixed margin loss in Sec \ref{sec:marginPO} to a dynamic, instance-specific formulation. Consequently, we propose $\gamma$-PO, a method that adapts the margin for each instance. The final policy optimization objective is consequently defined as:
\begin{equation}
\label{equ:loss}
    \mathcal{L}_\text{$\gamma$-PO} = -\mathbb{E}_{\mathcal{D}}[\log\sigma(r_w - r_l - \gamma_i)].
\end{equation}
The proposed $\gamma$-PO framework demonstrates remarkable generalizability as an algorithm-agnostic method. It can be seamlessly integrated with existing approaches, yielding enhanced variants such as $\gamma$-DPO:
\begin{equation}
\begin{aligned}
    &\mathcal{L}_\text{$\gamma$-DPO} = -\mathbb{E}_{\mathcal{D}}[\log\sigma (r_w - r_l - \gamma_i)], \quad \text{where} \\
    &r_w = \beta\log\frac{\pi_\theta(y_w|x)}{\pi_\text{ref}(y_w|x)}, \; r_l = \beta\log\frac{\pi_\theta(y_l|x)}{\pi_\text{ref}(y_l|x)},
\end{aligned}
\end{equation}
and $\gamma$-SimPO:
\begin{equation}
\begin{aligned}
    &\mathcal{L}_\text{$\gamma$-SimPO} = -\mathbb{E}_{\mathcal{D}}[\log\sigma  (r_w - r_l - \gamma_i)], \quad \text{where}\\
    &r_w = \frac{\beta}{|y_w|}\log\pi_\theta(y_w|x), \; r_l = \frac{\beta}{|y_w|}\log\pi_\theta(y_w|x).
\end{aligned}
\end{equation}

\input{table/pseudocode}

\subsection{Relation of $\gamma$-PO and rDPO}
To bridge $\gamma$-PO with existing robust methods, we analyze its connection to rDPO through their loss formulations. The rDPO objective applies label smoothing to the standard DPO loss:

\begin{equation}
\begin{aligned}
\mathcal{L}_\text{rDPO} = -\varepsilon\, &\mathbb{E}_\mathcal{D}[\log\sigma(r_l - r_w - \gamma_0)] \\
- (1 - \varepsilon) \, &\mathbb{E}_\mathcal{D} [\log\sigma(r_w - r_l - \gamma_0)],
\end{aligned}
\end{equation}

\noindent where $\varepsilon$ controls label smoothing intensity. Our key insight is that $\gamma$-PO implicitly implements \textit{adaptive label smoothing} through dynamic margins. We formalize this connection via:
\begin{theorem}
\label{thm:Approximation}
Let $\delta = \gamma_i - \gamma_0$ and $m = r_w - r_l$. When $|\delta| \ll |m|$, equating $\mathcal{L}_\text{rDPO}$ and $\mathcal{L}_\text{$\gamma$-PO}$ yields the approximation:
\begin{equation}
\varepsilon \approx \frac{\delta [1 - \sigma(m - \gamma_0)]}{\log(1 + e^{m+\gamma_0}) - \log(1 + e^{\gamma_0-m})}.
\end{equation}
\end{theorem}

This analysis demonstrates that $\gamma$-PO inherently implements \textit{adaptive label smoothing} through its dynamic margins. The derived relationship between $\delta$ (margin adjustment) and $\varepsilon$ (smoothing parameter) reveals that smaller reward margin $m = r_w - r_l$ automatically induce larger $\varepsilon$ values, effectively smoothing labels for ambiguous pairs. Conversely, a larger margin reduces $\varepsilon$, sharpening supervision for high-confidence samples. Crucially, the negative gradient $\partial\varepsilon/\partial d < 0$ formalizes this inverse relationship, enabling $\gamma$-PO to self-regulate label certainty based on reward margin magnitudes. By contrast, rDPO's static $\varepsilon$ requires pre-defined label flip rates and lacks this instance-level adaptation. 
This perspective provides new insights into the interpretation of $\gamma$-PO. The detailed proof of this theorem is provided in Appendix \ref{apd:poofthm}.

%% file: table/pseudocode.tex
\begin{algorithm}[t]
\caption{$\gamma$-PO Algorithm}\label{tab:pseudocode}
\begin{algorithmic}[1] 
\State \textbf{Input:} SFT model $\pi_\text{SFT}$, dataset $\mathcal{D}$, total number of iterations $T$, learning rate $\alpha$,
\State Initial policy model $\pi_\theta=\pi_\text{SFT}$,
\For{$t=0$ to $T$}
    \State Sample a batch of tuples $(x,y_w,y_l)$ from $\mathcal{D}$,
    \State Calculate the rewards ($r_w, r_l$) via Table \ref{tab:marginPO},
    \State Calculate the adaptive margin $\gamma_i$ using Equation \eqref{equ:loss-gamma2},
    \State Update the policy parameters $\theta$ for $\pi_\theta$ by applying gradient descent to Equation \eqref{equ:loss}: $\theta_{t+1} = \theta_t - \alpha \nabla_{\theta_t} \mathcal{L}_{\text{$\gamma$-PO}}$.
\EndFor
\end{algorithmic}
\end{algorithm}

%% file: chapter/4-experiment.tex
\input{table/main_result_err_avg}

\input{table/plugin_result}

\section{Experiments}

This section starts with an overview of the experimental setup, including various base models, training datasets, evaluation methods, and baselines. We then present the experimental results and visualize the dynamic target reward margin along with the hyperparameter analysis.

\subsection{Experiment Setup}
\textbf{Base Model.} 
Our approach is tested on four different base models: LLaMA-3-8B-Instruct \cite{grattafiori2024llama3}, Mistral-7B-Instruct \cite{jiang2023mistral7b}, Gemma-2-9B-Instruct \cite{gemmateam2024gemma2}, and Qwen-2.5-7B-Instruct \cite{qwen2025qwen25}. The diversity of these base models provides a solid foundation for our experiments.

\noindent\textbf{Dataset.} 
Following the approach in SimPO \cite{meng2024simpo}, a strong reward model is used to re-annotate the preference datasets. In this work, five different responses are generated for each prompt from the Ultrafeedback Binarized dataset\footnote{\href{https://huggingface.co/datasets/HuggingFaceH4/ultrafeedback_binarized}{HuggingFaceH4/ultrafeedback\_binarized}} using the base models. These responses are then re-annotated using a more powerful reward model, RLHFlow/ArmoRM-Llama3-8B-v0.1\footnote{\href{https://huggingface.co/RLHFlow/ArmoRM-Llama3-8B-v0.1}{RLHFlow/ArmoRM-Llama3-8B-v0.1}}. Then, the responses with the highest and lowest scores are selected to combine the preference data.

\noindent\textbf{Evaluation.}
We use three well-established metrics, AlpacaEval2, Arena-Hard, and MT-Bench to assess LLMs alignment with human preferences.
AlpacaEval2 \cite{dubois2024LCalpacaeval} is an enhanced automated metric that evaluates LLMs performance using Win Rate (WR) and Length-Controlled Win Rate (LC) with a dataset of 805 instructions. Both the reference model and judge model are GPT-4-1106-preview~\cite{openai2024gpt4technicalreport}. The length control improves robustness to length manipulation and better aligns with human evaluations.
Arena-Hard~\cite{arenahard2024_ArenaHard} assesses LLMs by applying 500 prompts derived from the Chatbot Arena dataset. It benchmarks the LLMs against GPT-4-0314~\cite{openai2024gpt4technicalreport} as the baseline, and GPT-4-1106-preview evaluates the responses. Each response is rated using a five-point Likert scale~\cite{liang2020Likert}, which facilitates a fair comparison and emphasizes performance disparities between models.
MT-bench~\cite{mt_bench} evaluates chatbots by using 80 multi-turn dialogue prompts that are specifically designed to assess their abilities in multi-round interactions and following instructions. It benchmarks the chatbots against GPT-4, which serves as both the baseline and the evaluator for the responses. Each response is scored based on specific criteria, allowing for objective comparisons and highlighting the differences in how well various models handle complex conversational tasks.

\noindent\textbf{Baselines.} 
We compare $\gamma$-PO with several state-of-the-art preference optimization methods, including DPO \cite{rafailov2024DPO}, IPO \cite{azar2023IPO}, CPO \cite{xu2024CPO}, KTO \cite{ethayarajh2024KTO}, ORPO \cite{hong2024ORPO}, R-DPO \cite{park2024R-DPO}, rDPO \cite{chowdhury2024rDPO}, $\beta$-DPO \cite{wu2024beta_DPO}, and SimPO \cite{meng2024simpo}. We also use SFT model as a baseline.
Further information regarding the baselines and training details can be found in Appendix \ref{apd:expdetail}.

\input{table/noisy_result}

\subsection{Experiment Results}
\textbf{$\gamma$-PO consistently outperforms the baselines.}  
We apply our method to one of the leading baseline preference optimization approaches, SimPO, resulting in $\gamma$-SimPO. The results in Table \ref{tab:main_result_err} show that $\gamma$-SimPO consistently outperforms other methods in both AlpacaEval2 and Arena-Hard, demonstrating stable performance improvements. For instance, in AlpacaEval2, it boosts the performance of Qwen2.5-Instruct from 34.1\% to 39.0\% of LC, and from 33.3\% to 34.3\% of WR. In the Mistral-Instruct setting, $\gamma$-SimPO outperforms the baseline SimPO by 6.5\% (LC) and 7.1\% (WR) in AlpacaEval2, and enhances WR of Arena-Hard from 20.7\% to 22.9\%. On average, our method leads to a 4.4\% improvement. 
Additionally, we also evaluate our methods on MT-Bench, as shown in Appendix \ref{apd:mt_bench}. The results presented in Table~\ref{tab:mt_bench} demonstrate that our proposed methods, $\gamma$-DPO and $\gamma$-SimPO, consistently achieve superior performance compared to baseline methods across all models. 
These results highlight the effectiveness of our approach in better aligning LLMs with human preferences. 

\noindent\textbf{$\gamma$-PO outperforms other plug-and-play methods.}
We compared $\gamma$-PO with other plug-and-play techniques, such as rDPO and $\beta$-DPO. Specifically, we applied the $\beta$-DPO and rDPO to SimPO, resulting in $\beta$-SimPO and rSimPO, respectively. Additionally, we apply the $\gamma$-PO approach to both DPO and SimPO, leading to the development of $\gamma$-DPO and $\gamma$-SimPO. 
The experimental results in Table \ref{tab:plugin_result} show that our method significantly outperforms both rDPO and $\beta$-DPO in improving model performance.
For instance, in the Qwen2.5-Instruct scenario, $\gamma$-SimPO exceeds the baseline SimPO by 6.1 in AlpacaEval2 and 5.8 in Arena-Hard.
Moreover, both $\gamma$-DPO and $\gamma$-SimPO achieved the highest average scores across all metrics and settings, with just one exception. 
Overall, $\gamma$-PO demonstrates a more effective enhancement of alignment performance compared to other plug-and-play methods.

\noindent\textbf{$\gamma$-PO handles flipped data better.}
To evaluate the resilience of the $\gamma$-SimPO method in the presence of label uncertainty, we deliberately introduce label noise into the training dataset by flipping labels at random rates of 10\% and 20\%. 
Building on the optimal baseline provided by the SimPO method, we conducted comparative experiments using rSimPO, $\beta$-SimPO, and our proposed $\gamma$-SimPO, all utilizing the Mistral-7B-Instruct as the base model. 
As shown in Table \ref{tab:noise_result}, $\gamma$-SimPO outperforms the other methods, achieving the highest average score, which is further illustrated in Figure \ref{fig:average_bar}. These results demonstrate that $\gamma$-SimPO exhibits robustness in scenarios involving noise, such as label flipping.

\input{table/GPU_time}

\begin{figure}[t]
  \includegraphics[width=\columnwidth]{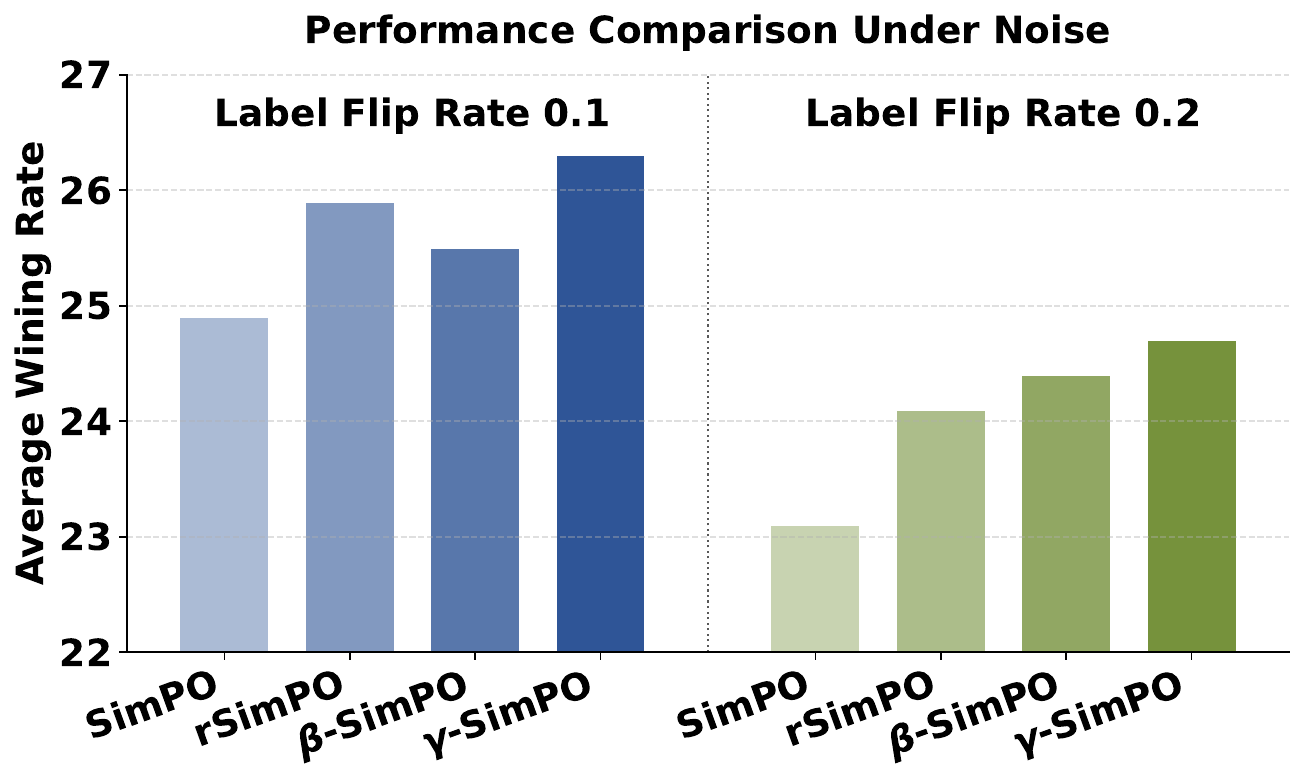}
  \caption{Comparison of average winning rates with randomly flipped labels at different probabilities.}
  \label{fig:average_bar}
\end{figure}

\noindent\textbf{$\gamma$-PO requires no additional computing overhead.}
To evaluate the computational efficiency of $\gamma$-PO, we measure the training duration when combining $\gamma$-PO with DPO and SimPO across four base models. As shown in Table \ref{tab:GPU_time}, $\gamma$-PO maintains computational efficiency comparable to the baseline methods while also delivering performance gains. Specifically, $\gamma$-DPO has only a 0.7\% longer training time than DPO, while $\gamma$-SimPO demonstrates an even more minor increase of 0.5\% over SimPO. These results confirm that our method performs better without significantly increasing GPU resource consumption.

\subsection{Visualization of Dynamic Target Margin}
Figure \ref{fig:visualize-gamma} depicts the relationship between the target margin $\gamma_i$ and the reward gap for various LLMs (Llama, Mistral, Gemma, and Qwen, which stand for Llama3-Instruct (8B), Mistral-Instruct (7B), Gemma2-Instruct (9B), and Qwen2.5-Instruct (7B), respectively).
The target margin exhibits a sigmoid-like increase with the reward gap, with the rapid upward slope closely corresponding to the range of the dynamic target margin.
This behavior suggests that the target margin is increased when the reward gap exceeds the threshold, prioritizing learning from high-gap pairs. In contrast, when the reward gap is smaller than the threshold, corresponding to low-gap pairs that may contain noise, the target margin is reduced to prevent overfitting to these pairs.
Furthermore, we observed that the target margin does not exhibit extreme values despite large or small reward gaps; it varies within a controlled, reasonable range. 
This visualization confirms the effectiveness of our approach in dynamically adjusting the target margin based on the reward gap, illustrating how it can be tailored to enhance learning from high-gap pairs while minimizing the influence of unreliable pairs.

\begin{figure*}[t]
  \includegraphics[width=1\textwidth]{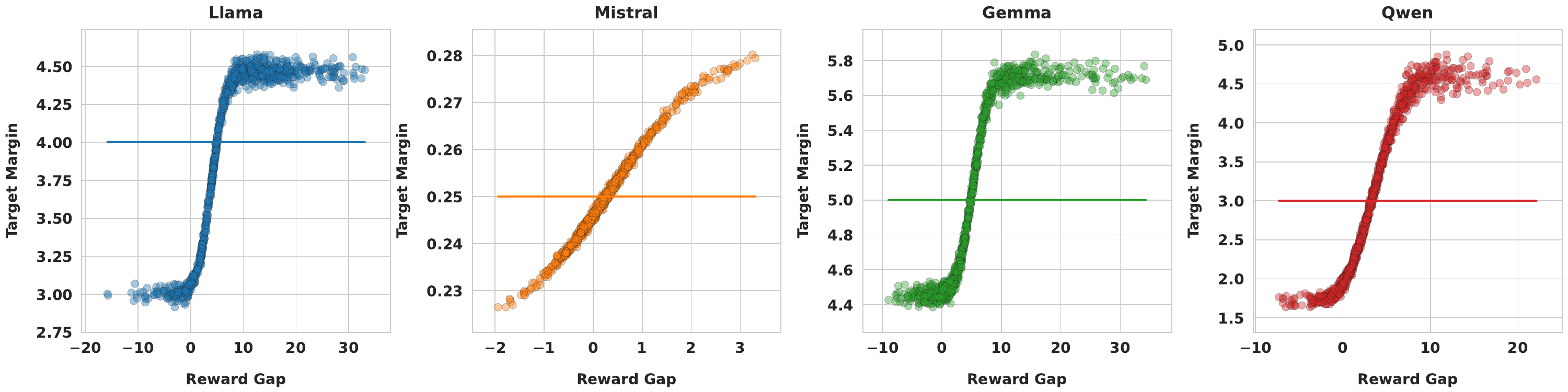}
  \caption{Visualization of dynamic target margin ($\gamma_i$) with reward gaps. The horizontal line indicates the initial value of target margin ($\gamma_0$).}
  \label{fig:visualize-gamma}
\end{figure*}

\subsection{Hyperparameter}

\begin{figure}[t]
  \includegraphics[width=\columnwidth]{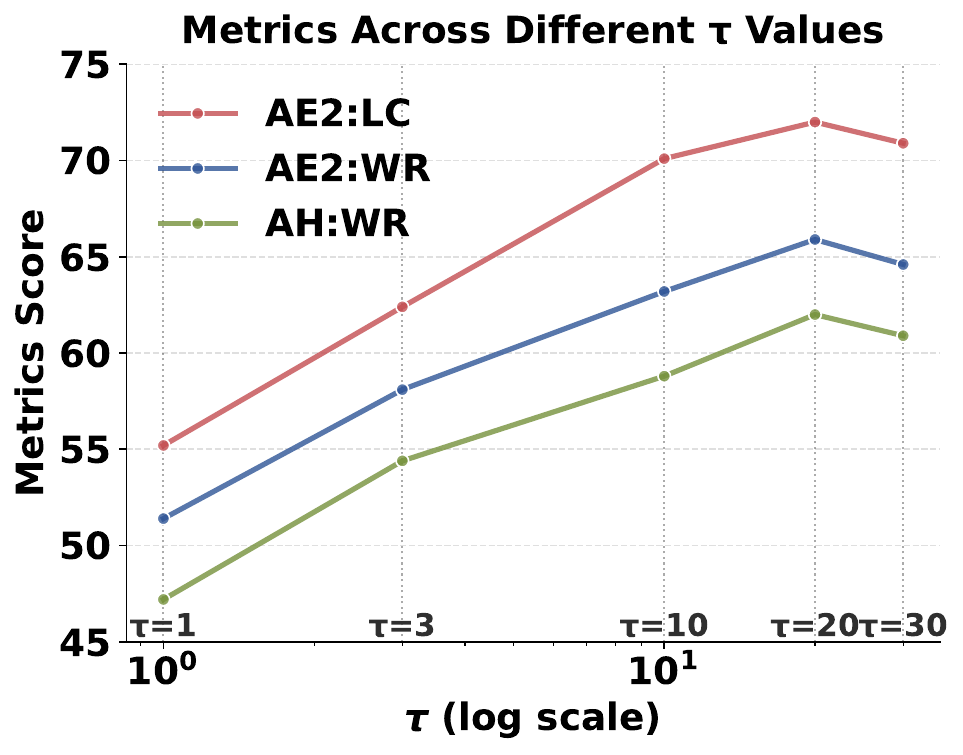}
  \caption{Performance metrics across different $\tau$ settings for the Gemma model, showing Length-Controlled Winning Rate (AE2:LC) and Raw Winning Rate (AE2:WR) from AlpacaEval2, and Raw Winning Rate (AH:WR) from Arena Hard.}
  \label{fig:hyperparameter}
\end{figure}

We thoroughly investigate how the hyperparameter $\tau$ in $\gamma$-PO affects the performance of AlpacaEval2 and Arena-Hard. Figure \ref{fig:hyperparameter} illustrates the performance metrics across different $\tau$ values for the Gemma base model. We plot three key metrics: Length-Controlled Winning Rate (AE2:LC) and Raw Winning Rate (AE2:WR) from AlpacaEval2, and Raw Winning Rate from Arena Hard (AH:WR). The results demonstrate a clear upward trend for all three metrics as $\tau$ increases when $\tau < 20$, followed by a decline when $\tau > 20$. This suggests that the optimal value of $\tau$ is almost the same across all metrics. 
The range for hyperparameter search is detailed in Table \ref{tab:hyperparams_baseline}, while the final optimal hyperparameters are presented in Table~\ref{tab:hyperparams_xpo}.

%% file: table/main_result_err_avg.tex
\setlength{\tabcolsep}{3pt}\label{main_result_err}
\begin{table*}[!t]
\centering
\resizebox{1.0\textwidth}{!}{
\begin{tabular}{lcccccccccccc}
\toprule
\multirow{4}{*}{\textbf{Method}} & \multicolumn{6}{c}{\textbf{Llama3-Instruct (8B)}} & \multicolumn{6}{c}{\textbf{Mistral-Instruct (7B)}} \\ 
\cmidrule(lr){2-7}\cmidrule(lr){8-13}
& \multicolumn{3}{c}{\textbf{AlpacaEval 2}} & \multicolumn{2}{c}{\textbf{Arena-Hard}} & \multicolumn{1}{c}{\textbf{Average}} & \multicolumn{3}{c}{\textbf{AlpacaEval 2}} & \multicolumn{2}{c}{\textbf{Arena-Hard}} & \multicolumn{1}{c}{\textbf{Average}}\\
\cmidrule(lr){2-4}\cmidrule(lr){5-6}\cmidrule(lr){7-7}\cmidrule(lr){8-10}\cmidrule(lr){11-12}\cmidrule(lr){13-13}
& LC(\%) & WR(\%) & err(\%) & WR(\%) & 95\%CI & & LC(\%) & WR(\%) & err(\%) & WR(\%) & 95\%CI & \\
\midrule
SFT & 28.2 & 27.4 & 1.38 & 24.3 & (-1.7, 1.9) & 26.6 & 21.3 & 16.0 & 1.15 & 13.1 & (-1.5, 1.1) & 16.8 \\ 
\midrule
IPO & 38.3 & 35.9 & 1.48 & 31.4 & (-2.3, 2.5) & 35.2 & 24.0 & 19.3 & 1.21 & 16.3 & (-1.6, 1.5) & 19.9 \\ 
KTO & 34.9 & 32.5 & 1.45 & 27.8 & (-2.8, 3.5) & 31.7 & 18.9 & 20.2 & 1.23 & 14.4 & (-1.6, 1.6) & 17.8 \\ 
CPO & 27.9 & 26.8 & 1.36 & 25.4 & (-2.3, 2.9) & 26.7 & 25.3 & 22.6 & 1.30 & 18.4 & (-1.7, 1.9) & 22.1 \\
ORPO & 32.0 & 30.3 & 1.41 & 25.9 & (-1.9, 2.4) & 29.4 & 24.0 & 21.8 & 1.27 & 20.4 & (-2.1, 2.0) & 22.1 \\
R-DPO & 47.9 & 43.5 & 1.48 & 32.9 & (-2.5, 2.7) & 41.4 & 24.8 & 22.0 & 1.27 & 17.1 & (-1.6, 1.7) & 21.3 \\
rDPO & 47.8 & 43.7 & 1.48 & 34.4 & (-2.2, 2.4) & 42.0 & 23.7 & 20.9 & 1.25 & 16.3 & (-1.5, 1.7) & 20.3 \\
DPO & 46.8 & 42.7 & 1.48 & 32.7 & (-2.6, 2.7) & 40.7 & 22.5 & 20.2 & 1.20 & 17.1 & (-2.0, 1.6) & 19.9 \\
$\beta$-DPO & 48.2 & 44.0 & 1.48 & \underline{36.4} & (-2.0, 2.1) & 42.9 & 23.3 & 20.6 & 1.22 & 17.5 & (-1.4, 1.6) & 20.5 \\
SimPO & \underline{54.3} & \underline{46.6} & 1.53 & 31.5 & (-2.4, 3.0) & \underline{44.1} & \underline{29.3} & \underline{30.8} & 1.38 & \underline{20.7} & (-2.0, 1.8) & \underline{26.9} \\
$\gamma$-SimPO & \textbf{55.4} & \textbf{48.7} & 1.48 & \textbf{37.3} & (-2.6, 2.8) & \textbf{47.1} & \textbf{31.2} & \textbf{33.0} & 1.39 & \textbf{22.9} & (-1.7, 2.2) & \textbf{29.0} \\
\midrule
Improve & \color{+}{+2.0\%} & \color{+}{+4.5\%} & - & \color{+}{+6.8\%} & - & \color{+}{+2.5\%} & \color{+}{+6.5\%} & \color{+}{+7.1\%} & - & \color{+}{+10.6\%} & - & \color{+}{+7.8\%}\\
\midrule
\multirow{4}{*}{\textbf{Method}} & \multicolumn{6}{c}{\textbf{Gemma2-Instruct (9B)}} & \multicolumn{6}{c}{\textbf{Qwen2.5-Instruct (7B)}} \\ 
\cmidrule(lr){2-7}\cmidrule(lr){8-13}
& \multicolumn{3}{c}{\textbf{AlpacaEval 2}} & \multicolumn{2}{c}{\textbf{Arena-Hard}} & \multicolumn{1}{c}{\textbf{Average}} & \multicolumn{3}{c}{\textbf{AlpacaEval 2}} & \multicolumn{2}{c}{\textbf{Arena-Hard}} & \multicolumn{1}{c}{\textbf{Average}} \\
\cmidrule(lr){2-4}\cmidrule(lr){5-6}\cmidrule(lr){7-7} \cmidrule(lr){8-10}\cmidrule(lr){11-12}\cmidrule(lr){13-13} 
& LC (\%) & WR (\%) & err (\%) & WR (\%) & 95\%CI (\%) & & LC (\%) & WR (\%) & err (\%) & WR (\%) & 95\%CI (\%) & \\
\midrule
SFT & 50.2 & 38.1 & 1.44 & 40.8 & (-2.3, 2.1) & 43.0 & 23.2 & 22.2 & 1.29 & 35.9 & (-3.1, 3.1) & 27.1 \\
\midrule
IPO & 61.3 & 57.2 & 1.45 & 56.7 & (-2.8, 2.6) & 58.4 & 25.1 & 25.4 & 1.34 & 43.7 & (-4.4, 2.5) & 31.4 \\ 
KTO & 60.7 & 54.5 & 1.46 & 51.8 & (-3.3, 2.8) & 55.7 & 32.6 & 30.5 & 1.38 & 46.6 & (-3.6, 3.0) & 36.6 \\ 
CPO & 50.3 & 38.8 & 1.45 & 41.8 & (-2.4, 2.1) & 43.6 & 23.3 & 19.5 & 1.29 & 38.3 & (-3.1, 3.1) & 27.0 \\ 
ORPO & 54.9 & 46.4 & 1.48 & 44.5 & (-2.3, 2.2) & 48.6 & 26.4 & 23.0 & 1.29 & 38.3 & (-3.2, 2.7) & 29.2 \\ 
R-DPO & 66.9 & \textbf{\underline{65.9}} & 1.37 & 60.9 & (-2.3, 2.8) & 64.6 & 33.2 & 31.7 & 1.39 & 53.1 & (-3.6, 3.5) & 39.3 \\ 
rDPO & 68.2 & \textbf{\underline{65.9}} & 1.39 & 59.1 & (-2.4, 2.9) & 64.4 & \underline{34.1} & \underline{33.3} & 1.42 & \textbf{65.2} & (-2.0, 2.6) & \underline{44.2} \\
DPO & 67.8 & 65.8 & 1.39 & 60.3 & (-2.0, 2.3) & 64.6 & 32.3 & 31.8 & 1.40 & 54.3 & (-3.7, 2.4) & 39.5 \\ 
$\beta$-DPO & 68.0 & \textbf{\underline{65.9}} & 1.37 & \underline{61.0} & (-2.2, 2.2) & \underline{65.0} & 32.4 & 31.4 & 1.39 & 61.0 & (-2.2, 2.2) & 41.6 \\ 
SimPO & \underline{70.9} & 64.5 & 1.43 & 55.3 & (-2.4, 2.0) & 63.6 & 33.3 & 28.2 & 1.37 & 59.0 & (-2.9, 3.0) & 40.2 \\ 
$\gamma$-SimPO & \textbf{72.0} & \textbf{65.9} & 1.42 & \textbf{62.0} & (-2.5, 1.8) & \textbf{66.6} & \textbf{39.0} & \textbf{34.3} & 1.46 & \underline{64.8} & (-2.1, 2.3) & \textbf{46.0} \\ 
\midrule
Improve & \color{+}{+1.6\%} & \color{+}{+0.0\%} & - & \color{+}{+1.6\%} & - & \color{+}{+2.5\%} & \color{+}{+14.4\%} & \color{+}{+3.0\%} & - & \color{-}{-0.6\%} & - & \color{+}{+4.1\%} \\
\bottomrule
\end{tabular}
}
\caption{The results from AlpacaEval and Arena-Hard under four different settings are shown. ``LC'' and ``WR'' represent the Length-Controlled and Raw Win Rates, respectively. \textbf{Bold} indicates the best performance for each metric, while \underline{underlined} values represent the best performance excluding our methods (our method is $\gamma$-SimPO, $i.e.$, our plug-and-play modules applied to SimPO). ``err'' refers to the standard error, while ``95\% CI'' stands for the 95\% confidence interval. ``Improve'' denotes the percentage improvement of the bold value over the underlined one. Each metric, except for ``err'' and ``95\% CI'', is considered better when it has a higher value.}
\label{tab:main_result_err}
\end{table*}
\setlength{\tabcolsep}{2pt}

%% file: table/plugin_result.tex
\setlength{\tabcolsep}{3pt}
\begin{table*}[!t]
\centering
\tiny
\resizebox{1.0\textwidth}{!}{
\begin{tabular}{lcccccccc}
\toprule
\multirow{4}{*}{\textbf{Method}} & \multicolumn{4}{c}{\textbf{Llama3-Instruct (8B)}} & \multicolumn{4}{c}{\textbf{Mistral-Instruct (7B)}} \\ 
\cmidrule(lr){2-5}\cmidrule(lr){6-9}
& \multicolumn{2}{c}{\textbf{AlpacaEval 2}} & \multicolumn{1}{c}{\textbf{Arena-Hard}} & \multicolumn{1}{c}{\textbf{Average}} & \multicolumn{2}{c}{\textbf{AlpacaEval 2}} & \multicolumn{1}{c}{\textbf{Arena-Hard}} & \multicolumn{1}{c}{\textbf{Average}} \\
\cmidrule(lr){2-3}\cmidrule(lr){4-4}\cmidrule(lr){5-5}\cmidrule(lr){6-7}\cmidrule(lr){8-8}\cmidrule(lr){9-9}
& LC(\%) & WR(\%) & WR(\%) & Avg. (rank) & LC(\%) & WR(\%) & WR(\%) & Avg. (rank) \\
\midrule
DPO & 46.8$^\text{+0.0}$ & 42.7$^\text{+0.0}$ & 32.7$^\text{+0.0}$ & 40.7 (4) & 
      22.5$^\text{+0.0}$ & 20.2$^\text{+0.0}$ & 17.1$^\text{+0.0}$ & 19.9 (4) \\
rDPO & 47.8$^{\color{+}\text{+1.0}}$ & 43.7$^{\color{+}\text{+1.0}}$ & 34.4$^{\color{+}\text{+1.7}}$ & 42.0 (2) & 
       23.7$^{\color{+}\text{+1.2}}$ & 20.9$^{\color{+}\text{+0.7}}$ & \textbf{17.5}$^{\color{+}\text{+0.4}}$ & 20.7 (2) \\
$\beta$-DPO & \textbf{48.2}$^{\color{+}\text{+1.4}}$ & \textbf{44.0}$^{\color{+}\text{+1.3}}$ & \textbf{36.4}$^{\color{+}\text{+3.7}}$ & \textbf{42.9 (1)} & 
       23.3$^{\color{+}\text{+0.8}}$ & 20.6$^{\color{+}\text{+0.4}}$ & \textbf{17.5}$^{\color{+}\text{+0.4}}$ & 20.5 (3) \\
$\gamma$-DPO & 47.8$^{\color{+}\text{+1.0}}$ & 43.0$^{\color{+}\text{+0.3}}$ & 35.1$^{\color{+}\text{+2.4}}$ & 42.0 (2) &
       \textbf{23.8}$^{\color{+}\text{+1.3}}$ & \textbf{21.7}$^{\color{+}\text{+1.5}}$ & \textbf{17.5}$^{\color{+}\text{+0.4}}$ & \textbf{21.0 (1)}\\
\midrule
SimPO & 54.3$^\text{+0.0}$ & 46.6$^\text{+0.0}$ & 31.5$^\text{+0.0}$ & 44.1 (4) & 
        29.3$^\text{+0.0}$ & 30.8$^\text{+0.0}$ & 20.7$^\text{+0.0}$ & 26.9 (4) \\
rSimPO & 53.0$^{\color{-}\text{-1.3}}$ & 45.5$^{\color{-}\text{-1.1}}$ & 36.5$^{\color{+}\text{+5.0}}$ & 45.0 (3) &
         28.9$^{\color{-}\text{-0.4}}$ & 30.3$^{\color{-}\text{-0.5}}$ & 22.2$^{\color{+}\text{+1.5}}$ & 27.1 (3) \\
$\beta$-SimPO & 54.8$^{\color{+}\text{+0.5}}$ & 48.1$^{\color{+}\text{+1.5}}$ & 32.6$^{\color{+}\text{+1.1}}$ & 45.2 (2) &
         29.4$^{\color{+}\text{+0.1}}$ & 31.1$^{\color{+}\text{+0.3}}$ & 22.6$^{\color{+}\text{+1.9}}$ & 27.7 (2) \\ 
$\gamma$-SimPO & \textbf{55.4}$^{\color{+}\text{+1.1}}$ & \textbf{48.7}$^{\color{+}\text{+2.1}}$ & \textbf{37.3}$^{\color{+}\text{+5.8}}$ & \textbf{47.1 (1)} &
         \textbf{31.2}$^{\color{+}\text{+1.9}}$ & \textbf{33.0}$^{\color{+}\text{+2.2}}$ & \textbf{22.9}$^{\color{+}\text{+2.2}}$ & \textbf{29.0 (1)} \\
\midrule
\multirow{4}{*}{\textbf{Method}} & \multicolumn{4}{c}{\textbf{Gemma2-Instruct (9B)}} & \multicolumn{4}{c}{\textbf{Qwen2.5-Instruct (7B)}} \\ 
\cmidrule(lr){2-5}\cmidrule(lr){6-9}
& \multicolumn{2}{c}{\textbf{AlpacaEval 2}} & \multicolumn{1}{c}{\textbf{Arena-Hard}} & \multicolumn{1}{c}{\textbf{Average}} & \multicolumn{2}{c}{\textbf{AlpacaEval 2}} & \multicolumn{1}{c}{\textbf{Arena-Hard}} & \multicolumn{1}{c}{\textbf{Average}} \\
\cmidrule(lr){2-3}\cmidrule(lr){4-4}\cmidrule(lr){5-5}\cmidrule(lr){6-7}\cmidrule(lr){8-8}\cmidrule(lr){9-9}
& LC(\%) & WR(\%) & WR(\%) & Avg. (rank) & LC(\%) & WR(\%) & WR(\%) & Avg. (rank) \\
\midrule
DPO & 67.8$^\text{+0.0}$ & 65.8$^\text{+0.0}$ & 60.3$^\text{+0.0}$ & 64.6 (3) &
      32.3$^\text{+0.0}$ & 31.8$^\text{+0.0}$ & 54.3$^\text{+0.0}$ & 39.5 (4) \\ 
rDPO & 68.2$^{\color{+}\text{+0.4}}$ & 65.9$^{\color{+}\text{+0.1}}$ & 59.1$^{\color{-}\text{-1.2}}$ & 64.4 (4) &
      34.1$^{\color{+}\text{+1.8}}$ & 33.3$^{\color{+}\text{+1.5}}$ & 65.2$^{\color{+}\text{+10.9}}$ & 44.2 (2) \\ 
$\beta$-DPO & 68.0$^{\color{+}\text{+0.2}}$ & 65.9$^{\color{+}\text{+0.1}}$ & 61.0$^{\color{+}\text{+0.7}}$ & 65.0 (2) &
      32.4$^{\color{+}\text{+0.1}}$ & 31.4$^{\color{-}\text{-0.4}}$ & 61.0$^{\color{+}\text{+6.7}}$ & 41.6 (3) \\ 
$\gamma$-DPO & \textbf{70.6}$^{\color{+}\text{+2.8}}$ & \textbf{68.0}$^{\color{+}\text{+2.2}}$ & \textbf{64.1}$^{\color{+}\text{+3.8}}$ & \textbf{67.6 (1)} &
      \textbf{34.4}$^{\color{+}\text{+2.1}}$ & \textbf{34.3}$^{\color{+}\text{+2.5}}$ & \textbf{68.4}$^{\color{+}\text{+14.1}}$ & \textbf{45.7 (1)} \\ 
\midrule
SimPO & 70.9$^\text{+0.0}$ & 64.5$^\text{+0.0}$ & 55.3$^\text{+0.0}$ & 63.6 (4) &
        33.3$^\text{+0.0}$ & 28.2$^\text{+0.0}$ & 59.0$^\text{+0.0}$ & 40.2 (4) \\
rSimPO & 71.6$^{\color{+}\text{+0.7}}$ & 64.5$^{\color{+}\text{+0.0}}$ & 58.6$^{\color{+}\text{+3.3}}$ & 64.9 (2) &
        38.8$^{\color{+}\text{+5.5}}$ & 31.7$^{\color{+}\text{+3.5}}$ & 61.4$^{\color{+}\text{+2.4}}$ & 44.0 (2) \\ 
$\beta$-SimPO & 71.1$^{\color{+}\text{+0.2}}$ & 62.8$^{\color{-}\text{-1.7}}$ & 59.0$^{\color{+}\text{+3.7}}$ & 64.3 (3) &
        34.5$^{\color{+}\text{+1.2}}$ & 29.5$^{\color{+}\text{+1.3}}$ & 59.1$^{\color{+}\text{+0.1}}$ & 41.0 (3) \\
$\gamma$-SimPO & \textbf{72.0}$^{\color{+}\text{+1.1}}$ & \textbf{65.9}$^{\color{+}\text{+1.4}}$ & \textbf{62.0}$^{\color{+}\text{+6.7}}$ & \textbf{66.6 (1)} &
        \textbf{39.0}$^{\color{+}\text{+5.7}}$ & \textbf{34.3}$^{\color{+}\text{+6.1}}$ & \textbf{64.8}$^{\color{+}\text{+5.8}}$ & \textbf{46.0 (1)} \\ 
\bottomrule
\end{tabular}
}
\caption{The improvements in results from AlpacaEval2 and Arena-Hard across various plug-and-play methods. ``LC'' stands for Length-Controlled, while ``WR'' represents Raw Win Rates. The methods rDPO and rSimPO refer to the application of rDPO to the DPO and SimPO methods, respectively. And the method $\beta$-SimPO refer to the application of $\beta$-DPO to the SimPO method. The superscript indicates enhancements based on the baselines, such as DPO and SimPO, where $\color{+}\text{red}$ represents improvement and $\color{-}\text{blue}$ signifies degradation. Performance metrics displayed in \textbf{bold} indicate the best results among the plugins. For each metric, a higher value denotes better performance.}
\label{tab:plugin_result}
\end{table*}
\setlength{\tabcolsep}{2pt}

%% file: table/noisy_result.tex
\setlength{\tabcolsep}{3pt}
\begin{table*}[!t]
\centering
\resizebox{1.0\textwidth}{!}{
\begin{tabular}{lcccccccc}
\toprule
\multirow{4}{*}{\textbf{Method}} & \multicolumn{4}{c}{\textbf{Label Flip Rate 0.1}} & \multicolumn{4}{c}{\textbf{Label Flip Rate 0.2}} \\ 
\cmidrule(lr){2-5}\cmidrule(lr){6-9}
& \multicolumn{2}{c}{\textbf{AlpacaEval 2}} & \multicolumn{1}{c}{\textbf{Arena-Hard}} & \multicolumn{1}{c}{\textbf{Average}} & \multicolumn{2}{c}{\textbf{AlpacaEval 2}} & \multicolumn{1}{c}{\textbf{Arena-Hard}} & \multicolumn{1}{c}{\textbf{Average}} \\
\cmidrule(lr){2-3}\cmidrule(lr){4-4}\cmidrule(lr){5-5}\cmidrule(lr){6-7}\cmidrule(lr){8-8}\cmidrule(lr){9-9}
& LC(\%) & WR(\%) & WR(\%) & Avg. (rank) & LC(\%) & WR(\%) & WR(\%) & Avg. (rank) \\
\midrule
SimPO & 26.2$^{\text{+0.0}}$ & 28.6$^{\text{+0.0}}$ & 20.0$^{\text{+0.0}}$ & 24.9 (4) &
        24.5$^{\text{+0.0}}$ & 26.1$^{\text{+0.0}}$ & 18.7$^{\text{+0.0}}$ & 23.1 (4) \\
rSimPO & 27.1$^{\color{+}\text{+0.9}}$ & \textbf{28.7}$^{\color{+}\text{+0.1}}$ & 21.9$^{\color{+}\text{+1.9}}$ & 25.9 (2) &
        25.8$^{\color{+}\text{+1.3}}$ & 27.0$^{\color{+}\text{+0.9}}$ & 19.6$^{\color{+}\text{+0.9}}$ & 24.1 (3) \\
$\beta$-SimPO & 26.6$^{\color{+}\text{+0.4}}$ & 28.4$^{\color{-}\text{-0.2}}$ & 21.6$^{\color{+}\text{+1.6}}$ & 25.5 (3) &
        \textbf{26.3}$^{\color{+}\text{+1.8}}$ & \textbf{27.9}$^{\color{+}\text{+1.8}}$ & 18.9$^{\color{+}\text{+0.2}}$ & 24.4 (2) \\ 
$\gamma$-SimPO & \textbf{28.2}$^{\color{+}\text{+2.0}}$ & 28.3$^{\color{-}\text{-0.3}}$ & \textbf{22.4}$^{\color{+}\text{+2.4}}$ & \textbf{26.3 (1)} &
        25.9$^{\color{+}\text{+1.4}}$ & 26.3$^{\color{+}\text{+0.2}}$ & \textbf{21.9}$^{\color{+}\text{+3.2}}$ & \textbf{24.7 (1)} \\
\bottomrule
\end{tabular}
}
\caption{The results from AlpacaEval2 and Arena-Hard illustrate the performance of various plug-and-play methods in a noisy label setting. In this context, ``LC'' stands for Length-Controlled, and ``WR'' stands for Raw Win Rates. Metrics that are highlighted in \textbf{bold} indicate the best performance among the plugins evaluated. The base model used for this analysis is Mistral-7B-Instruct. For each metric, a higher value indicates better performance.}
\label{tab:noise_result}
\end{table*}
\setlength{\tabcolsep}{2pt}

%% file: table/GPU_time.tex
\begin{table*}
  \centering  
  \begin{tabular}{l@{\hspace{1.5mm}}|@{\hspace{1.5mm}}c@{\hspace{2.5mm}}c@{\hspace{1.5mm}}|@{\hspace{1.5mm}}c@{\hspace{2.5mm}}c}
    \toprule
    Base Model & DPO & $\gamma$-DPO & SimPO & $\gamma$-SimPO \\
    \midrule
    Llama & 4A100*4h54min & 4A100*4h55min$^{\color{+}{\text{+1min}}}$ & 4A100*3h30min & 4A100*3h33min$^{\color{+}{\text{+3min}}}$ \\
    Mistral & 4A100*4h45min & 4A100*4h53min$^{\color{+}{\text{+8min}}}$ & 4A100*3h41min & 4A100*3h42min$^{\color{+}{\text{+1min}}}$ \\
    Gemma & 8A100*4h04min & 8A100*4h03min$^{\color{-}{\text{-1min}}}$ & 8A100*3h13min & 8A100*3h12min$^{\color{-}{\text{-1min}}}$ \\
    Qwen & 4A100*5h26min & 4A100*5h29min$^{\color{+}{\text{+3min}}}$ & 4A100*3h33min & 4A100*3h36min$^{\color{+}{\text{+3min}}}$ \\
    \midrule
    Avg. GPU time & 23.22h$\cdot$A100 & 23.38h$\cdot$A100$^{\color{+}{\text{+0.7\%}}}$ & 17.17h$\cdot$A100 & 17.25h$\cdot$A100$^{\color{+}{\text{+0.5\%}}}$ \\
    \bottomrule
  \end{tabular}
  \caption{A100 GPU time consumption for training DPO, $\gamma$-DPO, SimPO, and $\gamma$-SimPO across all base models. }
  \label{tab:GPU_time}
\end{table*}

%% file: chapter/5-related_work.tex
\section{Related Work}
\subsection{Reinforcement Learning from Human Feedback}
Reinforcement Learning from Human Feedback (RLHF) \cite{christiano2017deep, ziegler2019fine, azar2024general} is a crucial post-training step following Supervised Fine-Tuning (SFT). 
RLHF helps align LLMs outputs with human values such as helpfulness, harmlessness, and honesty \cite{ouyang2022training, bai2022training}. 
It has been widely applied in the revolutionary closed-source LLMs, such as GPT-4 \cite{openai2024gpt4technicalreport} and Gemini \cite{team2023gemini}, as well as powerful open-source models like LLaMA-3 \cite{grattafiori2024llama3}, Gemma-2 \cite{gemmateam2024gemma2}, and Qwen-2.5 \cite{qwen2025qwen25}.
The RLHF process generally consists of three main steps.
First, a policy $\pi_\text{SFT}$ is trained on a large dataset through SFT~\cite{zhou2024lima-sft, ding2023enhancing-sft, kopf2024openassistant-sft}. Next, $\pi_\text{SFT}$ generates multiple responses, which are then labeled by humans into preference pairs. These pairs are used to train a reward model, $r_\theta$, that scores the model’s responses based on their alignment with human preferences \cite{chen2024odin-reward, gao2023scaling-reward, lightman2023let-reward, havrilla2024glore-reward, luo2023wizardmath-reward, lambert2024rewardbench-reward, dong2024rlhfworkflowrewardmodeling}. Lastly,
the reward model updates the policy model via PPO algorithm~\cite{schulman2017PPO} in RL \cite{anthony2017thinking-ppo}.

\subsection{Direct Preference Optimization}
Recent studies introduce several enhancements to DPO, including IPO~\cite{azar2023IPO}, KTO~\cite{ethayarajh2024KTO}, CPO~\cite{xu2024CPO}, ORPO~\cite{hong2024ORPO}, R-DPO~\cite{park2024R-DPO}, $\beta$-DPO~\cite{wu2024beta_DPO}, SimPO~\cite{meng2024simpo}, $\alpha$-DPO~\cite{wu2024alphadpo}, WPO~\cite{zhou2024wpo}, and
RainbowPO~\cite{zhao2024rainbowpo}.
In addition, Dr. DPO~\cite{drdpo} uses Distributionally Robust Optimization to handle noisy data and align human preferences.
\citet{reward_gap} introduce the alignment potential metric $M_{AP}$ to optimize preference data selection for LLM alignment.

%% file: chapter/6-conclusion.tex
\section{Conclusion}

In this work, we introduce a plug-and-play approach called $\gamma$-PO, which is applicable to DPO and its variants. 
The key idea behind $\gamma$-PO is personalizing the target margin based on the reward margin, increasing them for high-confidence pairs to prioritize learning from these instances while reducing them for low-reward margin pairs to mitigate the impact of uncertain preference, which improves the handling of ambiguous preference data and reduces the risk of noisy data distorting the alignment policy. 
Furthermore, $\gamma$-PO offers a simple and efficient solution with minimal impact on training efficiency, making it a valuable tool for improving alignment in practical applications.
Extensive experiments demonstrate that the $\gamma$-PO enhances the alignment performance, achieving robust results on AlpacaEval2 and Arena-Hard. 
Our findings highlight the potential of dynamic margin strategies as a promising direction for improving the alignment of LLMs, particularly in the presence of noisy data. There is still substantial opportunity for deeper exploration of this approach to enhance model alignment in real-world settings.

%% file: chapter/7-limitation.tex
\section*{Limitations}

Although the $\gamma$-PO method demonstrates significant improvements in LLMs alignment, it introduces several limitations that warrant further investigation and refinement in future work.

\noindent\textbf{Additional Hyperparameter.}
The $\gamma$-PO method, while enhancing performance, introduces an additional hyperparameter, $\tau$, which significantly impacts model alignment. Future efforts will focus on eliminating this parameter, aiming to compute the adaptive margin automatically using existing parameters.

\noindent\textbf{Limited Evaluation.}
In our experiments, we used two evaluation benchmarks: AlpacaEval2 \cite{li2023alpacaeval} and Arena-Hard \cite{arenahard2024_ArenaHard}. Although these benchmarks are widely used in the field \cite{wu2024alphadpo, meng2024simpo, wu2024beta_DPO}, we observed that their results are influenced by the choice of parameters (\textit{e.g.}, the version of `alpaca\_eval'). In the future, incorporating a broader range of LLMs evaluation metrics and developing benchmarks that are more robust to parameter variations could provide a more comprehensive assessment of model performance.

\section{Ethical Considerations}
While the UltraFeedback dataset \cite{cui2024ultrafeedback} provides high-quality preference and textual feedback, covering aspects such as instruction following, truthfulness, honesty, and helpfulness, it still faces challenges in addressing the safety of LLMs. This includes mitigating risks such as toxicity, stereotype bias, adversarial and out-of-distribution robustness, privacy, machine ethics, and fairness \cite{wang2024decodingtrust}. Future alignment efforts will require the integration of larger and more diverse training datasets to comprehensively address these issues \cite{ji2024beavertails, zhao2024wildchat}.

\section*{ACKNOWLEDGEMENT}
This work was supported by Ant Group Research Fund.

%% file: chapter/8-appendix.tex
\section{Proof of Theorem \ref{thm:Approximation}}\label{apd:poofthm}
\setcounter{theorem}{0}
\renewcommand{\thetheorem}{3.\arabic{theorem}}
\begin{theorem}
    \label{thm:Approximation}
    Consider the rDPO loss function defined as
    \begin{equation}
    \begin{aligned}
        \mathcal{L}_\text{rDPO} = -\varepsilon &\log\sigma(r_l - r_w - \gamma_0) \\
        - (1 - \varepsilon) &\log\sigma(r_w - r_l - \gamma_0),
    \end{aligned}
    \end{equation}
    where $\sigma(x) = \frac{1}{1 + e^{-x}}$ is the sigmoid function, and $\varepsilon$ is a smoothing parameter. 

    If we set $\mathcal{L}_\text{rDPO} = \mathcal{L}_\text{$\gamma$-PO}$ for simplicity, and denote $\delta = \gamma_i - \gamma_0$ and $m = r_w - r_l$, then under the condition $|\delta| \ll |m|$, the approximate solution for $\varepsilon$ is given by
    \begin{equation}
    \varepsilon \approx \frac{\delta \left[ 1 - \sigma(m - \gamma_0) \right]}{\log\left(1 + e^{m + \gamma_0}\right) - \log\left(1 + e^{\gamma_0-m}\right)}.
    \end{equation}
\end{theorem}

\begin{proof}
    Starting with the given equation:
    \begin{equation}
    \begin{aligned}
        &-\log\sigma(m - \gamma_0 - \delta) \\
        &= -\varepsilon\log\sigma(-m - \gamma_0) -(1 - \varepsilon)\log\sigma(m - \gamma_0),
    \end{aligned}
    \end{equation}
    using the definition of the sigmoid function $\sigma(x) = \frac{1}{1 + e^{-x}}$ and $-\log\frac{1}{x}=\log x$, we rewrite the equation as:
    \begin{equation}
    \begin{aligned}
        \log\left(1 + e^{-(m - \gamma_0 - \delta)}\right) = \varepsilon & \log\left(1 + e^{-(-m - \gamma_0)}\right) \\
        + (1 - \varepsilon) & \log\left(1 + e^{-(m - \gamma_0)}\right).
    \end{aligned}
    \end{equation}
    Assuming that $\delta$ is small, we apply a first-order Taylor expansion around $\delta = 0$ to approximate the left-hand side:
    \begin{equation}
    \begin{aligned}
        &\log\left(1 + e^{-(m - \gamma_0)}\right) + \delta \cdot \frac{e^{-(m - \gamma_0)}}{1 + e^{-(m - \gamma_0)}} \\
        \approx &(1 - \varepsilon) \log\left(1 + e^{\gamma_0-m}\right) + \varepsilon \log\left(1 + e^{m + \gamma_0}\right).
    \end{aligned}
    \end{equation}
    Rearranging the terms to solve for $\delta$, we get:
    \begin{equation}
    \begin{aligned}
        &\delta \cdot \frac{e^{\gamma_0-m}}{1 + e^{\gamma_0-m}} \\
        \approx &\varepsilon \left[ \log\left(1 + e^{m + \gamma_0}\right) - \log\left(1 + e^{\gamma_0-m)}\right) \right].
    \end{aligned}
    \end{equation}
    Recognizing that $\frac{e^{\gamma_0-m}}{1 + e^{\gamma_0-m}} = \sigma(\gamma_0-m) = 1 - \sigma(m - \gamma_0)$, we express $\delta$ as:
    \begin{equation}
    \begin{aligned}
        \varepsilon \approx \frac{\delta \left[ 1 - \sigma(m - \gamma_0) \right]}{\log\left(1 + e^{m + \gamma_0}\right) - \log\left(1 + e^{\gamma_0-m}\right)}.
    \end{aligned}
    \end{equation}
\end{proof}

In addition, as training approaches convergence, as shown in Figure \ref{fig:visualize-gamma}, most of the values of $\gamma_i$ lie near $\gamma_0$, while $\delta$ remains small under the control of $\tau$, thereby satisfying the assumption $|\delta| \ll |m|$.

\section{Experiment Details}\label{apd:expdetail}

\paragraph{Training hyperparameters.}
\input{table/hyperparams_baseline}
We observed that the alignment performance of LLMs is influenced by the choice of hyperparameters. To ensure a fair comparison, we conducted a hyperparameter search for all methods, with the search range detailed in Table \ref{tab:hyperparams_baseline}.
For the remaining hyperparameters, the learning rate was selected from the range [3e-7, 5e-7, 6e-7, 8e-7, 1e-6]. A consistent batch size of 128 was used across all methods. All models were trained for one epoch using a cosine learning rate schedule, which included a 10\% warm-up phase. The Adam optimizer was employed \citep{kingma2017adam}. Additionally, the maximum sequence length was set to 2048.

\noindent\textbf{Hyperparameter in $\gamma$-PO.}
Table \ref{tab:hyperparams_xpo} shows the hyperparameters of our method.
\input{table/hyperparams_XPO}

\noindent\textbf{Decoding hyperparameters.}
The decoding hyperparameters are consistent with those used in SimPO\footnote{\href{https://github.com/princeton-nlp/SimPO/tree/main/eval}{https://github.com/princeton-nlp/SimPO/tree/main/eval}}. We would like to express our sincere gratitude to the SimPO team for generously sharing their valuable insights.

\noindent\textbf{Computation environment.}
All training experiments in this paper were conducted using 8×A100 GPUs, following the procedures outlined in the alignment-handbook repository\footnote{\href{https://github.com/huggingface/alignment-handbook}{https://github.com/huggingface/alignment-handbook}}.

\section{Algorithm for calculating $\gamma$}
\label{apd:calculate-gamma}

Due to GPU memory constraints, the batch size is relatively small, which may introduce sampling bias, leading to significant variations in reward margins across different batches. To mitigate this issue, we introduce an additional queue to reduce bias.
Specifically, we implement a First-In-First-Out (FIFO) reward margin queue, denoted as $\mathcal{H}$. This mechanism maintains an updated record of recent reward margins, and incorporating more samples from $\mathcal{H}$ helps better approximate the true distribution of reward margins.
The revised pseudocode is provided in Algorithm \ref{apd:pseudocode}.

\input{table/pseudocode_queue}

For optimizing $\gamma_i$, we apply \href{https://nisheethvishnoi.wordpress.com/wp-content/uploads/2018/05/lecture42.pdf}{Mirror Descent and the Multiplicative Weight Update Method} to efficiently compute $\gamma$. The optimization problem for $\gamma$ is formulated as follows:
\begin{equation}
\begin{aligned}
    p =\arg \min_p & (-\mathbb{E}_\mathcal{D}[\log \sigma(m - \gamma_i) \\
    & - \tau \, \mathbb{D}_\text{KL}(p||p_0))],
\end{aligned}
\end{equation}
where $m$ represents the reward margins, $\gamma_i=m_n p_i \gamma_0$, $\sigma(x)$ is the sigmoid function, $p$ is the probability distribution over the reward gaps, $\gamma_0$ is the initial margin value, $p_0$ is the uniform distribution (used as the target), and $\tau$ is a regularization parameter controlling the KL divergence term. 
For simplicity, We denote $L=-\mathbb{E}_\mathcal{D}[\log \sigma(m - \gamma_i) + \tau \, \mathbb{D}_\text{KL}(p||p_0)]$.
The gradients of the loss with respect to $p$ are computed as:
\begin{equation}
\begin{aligned}
    \frac{\partial L}{\partial p_i} = \frac{\gamma_0}{1 + e^{m - m_n p_i \gamma_0}} + \tau\left(1 + \log(m_n p_i)\right),
\end{aligned}
\end{equation}
where $m_n$ is the number of reward margin $m$. These gradients are then used to iteratively update \(p\) using a multiplicative update rule:
\begin{equation}
    p_i \leftarrow p_i \cdot \frac{\exp(-\eta \cdot \text{grad}(p_i))}{\sum_i \exp(-\eta \cdot \text{grad}(p_i))},
\end{equation}
where $\eta$ is the learning rate.
Furthermore, we derive the second-order gradients with respect to  p  as:
\begin{equation}
\begin{aligned}
    \frac{\partial^2 L}{\partial p_i^2} &= \frac{m_n \gamma_0^2 e^{m - m_n p_i \gamma_0}}{\left( 1 + e^{m - m_n p_i \gamma_0} \right)^2} + \frac{\tau}{p_i} > 0,
\end{aligned}
\end{equation}
which indicates that it is a convex optimization problem and that an optimal solution for  p  exists.

In our implementation, we set $\eta = \min\left(\frac{0.5}{\tau}, 1\right)$ and $m_n = 256$. The optimization process is repeated $T=20$ times, which allows for sufficient convergence.

\section{Outputs Example}
\input{table/output_example}
We present sample outputs generated by the Gemma base model, as shown in Table \ref{tab:output_example}.


\input{table/notation}


\section{The artifacts}
\textbf{Base models.} In this manuscript, we use several publicly available resources, including four base models (LLaMA-3-8B-Instruct \cite{grattafiori2024llama3}, Mistral-7B-Instruct \cite{jiang2023mistral7b}, Gemma-2-9B-Instruct \cite{gemmateam2024gemma2}, and Qwen-2.5-7B-Instruct \cite{qwen2025qwen25}), which are licensed under their respective terms for research and non-commercial use. Additionally, we utilize the Ultrafeedback Binarized dataset\footnote{\href{https://huggingface.co/datasets/HuggingFaceH4/ultrafeedback_binarized}{HuggingFaceH4/ultrafeedback\_binarized}}, which is publicly available under the CC-BY-4.0 license for academic and research purposes. The reward model, RLHFlow/ArmoRM-Llama3-8B-v0.1 \footnote{\href{https://huggingface.co/RLHFlow/ArmoRM-Llama3-8B-v0.1}{RLHFlow/ArmoRM-Llama3-8B-v0.1}}, is also publicly available on Hugging Face with terms permitting non-commercial use.

\noindent\textbf{Evaluation.} We evaluate our methods using two datasets—AlpacaEval2 \cite{dubois2024LCalpacaeval} and Arena-Hard \cite{arenahard2024_ArenaHard}—both of which are publicly available and licensed for research purposes. All assets used in this work are provided with clear attribution and comply with the licensing terms of the respective sources. Full license details and terms of use for each artifact can be found in the accompanying repository and documentation.


\noindent\textbf{Dataset.} The original UltraFeedback dataset contains 64k prompts, each accompanied by four model completions from various open and proprietary models. GPT-4 was then used to assign scores to each completion based on criteria such as usefulness and honesty. To create UltraFeedback Binarized, we selected the highest overall score as the "chosen" completion and randomly selected one of the remaining three as the "rejected" completion. This defines the preference modeling split for techniques like reward modeling or DPO. We also created splits for supervised fine-tuning (SFT), using the "chosen" column as the conversation to be modeled, and for generations (such as rejection sampling or PPO). For detailed information on dataset processing, please refer to the accompanying scripts.



\section{More Evaluation}\label{apd:mt_bench}
We also evaluation our methods on MT-Bench~\cite{mt_bench}, results are shown in Table~\ref{tab:mt_bench}.
\input{table/mt_bench}

%% file: table/hyperparams_baseline.tex
\begin{table*}[!ht]
    \centering
    \resizebox{\textwidth}{!}{
    \begin{tabular}{lll}
    \toprule 
    \textbf{Method} & \textbf{Objective} & \textbf{Hyperparameter} \\ \midrule
    DPO~\cite{rafailov2024DPO} & $-\log \sigma \left( \beta \log \frac{\pi_\theta(y_w|x)}{\pi_{\text{ref}}(y_w|x)} - \beta \log \frac{\pi_\theta(y_l|x)}{\pi_{\text{ref}}(y_l|x)}\right)$ & $\beta \in [0.01, 0.05, 0.1]$ \\ \midrule 
    IPO~\cite{azar2023IPO} & $ \left( \log \frac{\pi_\theta(y_w|x)}{\pi_{\text{ref}}(y_w|x)} - \log \frac{\pi_\theta(y_l|x)}{\pi_{\text{ref}}(y_l|x)} - \frac{1}{2\tau} \right)^2$ & $\tau \in [0.01, 0.1, 0.5, 1.0]$ \\  \midrule 
    CPO~\cite{xu2024CPO} &  $-\log \sigma  \left(\beta \log \pi_\theta(y_w|x) - \beta \log \pi_\theta(y_l|x) \right) - \lambda \log \pi_\theta (y_w|x)$ & $\alpha = 1.0, \,\, \beta \in [0.01, 0.05, 0.1]$ \\ \midrule
    \multirow{2}{*}{KTO~\cite{ethayarajh2024KTO}} & $-\lambda_w \sigma \left( \beta \log \frac{\pi_\theta(y_w|x)}{\pi_{\text{ref}}(y_w|x)} - z_{\text{ref}} \right) +  \lambda_l \sigma \left( z_{\text{ref}} - \beta \log \frac{\pi_\theta(y_l|x)}{\pi_{\text{ref}}(y_l|x)} \right),\,$ & $\lambda_l = \lambda_w = 1.0$ \\  
    & $\text{where} \,\, z_{\text{ref}} = \mathbb{E}_{(x, y) \sim \mathcal{D}} \left[\beta \text{KL}\left( \pi_\theta(y|x) || \pi_{\text{ref}}(y|x) \right)  \right]$ & $\beta \in [0.01, 0.05, 0.1]$ \\ \midrule
    \multirow{2}{*}{ORPO~\cite{hong2024ORPO}} & $-\log p_\theta(y_w|x) - \lambda  \log \sigma \left(\log \frac{p_\theta(y_w|x)}{1 - p_\theta(y_w|x)} - \log \frac{p_\theta(y_l|x)}{1 - p_\theta(y_l|x)}  \right),\,$ & \multirow{2}{*}{$\lambda \in [0.1, 0.5, 1.0, 2.0]$} \\  
    & $\text{where} \,\, p_\theta(y|x) = \exp\left( \frac{1}{|y|} \log \pi_\theta(y|x) \right)$ \\  \midrule
    \multirow{2}{*}{R-DPO~\cite{park2024R-DPO}} & \multirow{2}{*}{$-\log \sigma \left( \beta \log \frac{\pi_\theta(y_w|x)}{\pi_{\text{ref}}(y_w|x)} - \beta \log \frac{\pi_\theta(y_l|x)}{\pi_{\text{ref}}(y_l|x)} - \left(\alpha |y_w| - \alpha |y_l| \right) \right)$} & $\alpha \in [0.05, 0.1, 0.5, 1.0]$ \\
    & & $\beta \in [0.01, 0.05, 0.1]$ \\
    \midrule 
    \multirow{2}{*}{rDPO~\cite{chowdhury2024rDPO}} & $-(1-\varepsilon)\log \sigma  \left(\text{rg} \right) - \varepsilon\log \sigma  \left(-\text{rg} \right)$ & $\beta \in [0.01, 0.05, 0.1]$ \\
    &$\text{where} \,\, \text{rg}=\beta\log\frac{\pi_\theta(y_w|x)}{\pi_\text{ref}(y_w|x)} - \beta\log\frac{\pi_\theta(y_l|x)}{\pi_\text{ref}(y_l|x)}$ & $\varepsilon \in [0.01, 0.05, 0.1]$ \\
    \midrule 
    \multirow{2}{*}{SimPO~\cite{meng2024simpo}} & \multirow{2}{*}{$-\log \sigma  \left( \frac{\beta}{|y_w|} \log \pi_\theta(y_w|x) - \frac{\beta}{|y_l|} \log \pi_\theta(y_l|x) - \gamma \right)$} & $\beta \in [2.5, 5.0, 10.0]$ \\
    & & $\gamma \in [0.3, 0.4, 0.5, 1.0]$ \\
    \midrule 
    \multirow{2}{*}{$\beta$-DPO~\cite{wu2024beta_DPO}} & {$-\log \sigma \left( \beta_i \log \frac{\pi_\theta(y_w|x)}{\pi_{\text{ref}}(y_w|x)} - \beta_i \log \frac{\pi_\theta(y_l|x)}{\pi_{\text{ref}}(y_l|x)}\right)$} & $\beta\in [2.5,10.0], \,\, \gamma \in [0.1, 0.3, 0.5]$  \\
    & $\text{where} \,\, \beta_i=\beta_0+\alpha(M_i-M_0)\beta_0, \,\, M=\beta_0 \log \frac{\pi_\theta(y_w|x)}{\pi_{\text{ref}}(y_w|x)} - \beta_0 \log \frac{\pi_\theta(y_l|x)}{\pi_{\text{ref}}(y_l|x)}$ & $\alpha \in [1e-2, 5e-2, 0.1, 0.2]$ \\
    \bottomrule
    \end{tabular}
    }
\caption{Various preference optimization objectives and hyperparameter search range.}
\label{tab:hyperparams_baseline}
\end{table*}

%% file: table/hyperparams_XPO.tex
\begin{table}[h!]
\centering
\begin{tabular}{lcccc}
\toprule 
\textbf{Setting} & $\beta$ & $\gamma$ & $\tau$ & Learning rate \\ \midrule
Llama3-Instruct & 10.0 & 0.4 & 10.0 & 1e-6 \\
Mistral-Instruct & 2.5 & 0.15 & 10.0 & 6e-7 \\
Gemma2-Instruct & 10.0 & 0.5 & 20.0 & 8e-7 \\
Qwen2.5-Instruct & 10.0 & 0.3 & 3.0 & 1e-6 \\
\bottomrule
\end{tabular}
\caption{The hyperparameter values in $\gamma$-PO used for each training setting.}
\label{tab:hyperparams_xpo}
\end{table}

%% file: table/pseudocode_queue.tex
\begin{algorithm}[t]
\caption{$\gamma$-PO Algorithm}\label{apd:pseudocode}
\begin{algorithmic}[1] 
\State \textbf{Input:} SFT model $\pi_\text{SFT}$, dataset $\mathcal{D}$, total number of iterations $T$, learning rate $\alpha$,
\State Initial policy model $\pi_\theta=\pi_\text{SFT}$, empty FIFO queue $\mathcal{H}$,
\For{$t=0$ to $T$}
    \State Sample a batch of tuples $(x, y_w, y_l)$ from $\mathcal{D}$.
    \State Sample a reward pair $(r_w, r_l)$ from $\mathcal{H}$.
    \State Calculate the rewards pair $(r_w, r_l)$ via Table \ref{tab:marginPO} and store them in $\mathcal{H}$.
    \State Compute the adaptive margin $\gamma_i$ using Equation \eqref{equ:loss-gamma2} with reward pairs $(r_w, r_l)$ both from $\mathcal{H}$  current batch.
    \State Extract the adaptive margin $\gamma_i$ from the current batch, calculate the loss $\mathcal{L}_{\gamma\text{-PO}}$ using Equation \eqref{equ:loss},
    \State Update the policy parameters $\theta$ for $\pi_\theta$ via gradient descent: $\theta_{t+1} = \theta_t - \alpha \nabla_{\theta_t} \mathcal{L}_{\gamma\text{-PO}}$.
\EndFor
\end{algorithmic}
\end{algorithm}

%% file: table/output_example.tex
\begin{table*}[ht]
\centering
\begin{tabular}{@{}p{0.18\linewidth}p{0.82\linewidth}@{}}
\toprule
Instruction: & Hi, I'm trying to solve a crossword puzzle, but I've never done one of these before. Can you help me out? \\
\midrule
DPO output: & I'd love to help you with your crossword puzzle!  Here's how we can break it down to make it easier:\textbackslash n\textbackslash n**1.  Tell me the specifics:**\textbackslash n\textbackslash n* **Clues:**  Share the  **across** and **down** clues you're stuck on.  Be as specific as possible about the number of letters each answer should have (e.g., \textbackslash"7 letters across\textbackslash").\textbackslash n* **Theme (if any):**  Some crosswords have a theme. Knowing it can give big hints.\textbackslash n* **Intersection Help:** If you have some letters already filled in,  let me know where they intersect with the clues you need help with. This can narrow down possibilities.\textbackslash n\textbackslash n**2.  I can offer:**\textbackslash n\textbackslash n* **Definitions and synonyms:**  I can provide definitions or related words to help you decipher clues.\textbackslash n* **Wordplay hints:** Crosswords often use puns, anagrams, or other wordplay. I can help spot those.\textbackslash n* **Letter suggestions:**  If you have partial answers, I might be able to suggest missing letters.\textbackslash n\textbackslash n\textbackslash nLet's get started!  What's your first clue? \\
\midrule
$\gamma$-SimPO output: & Absolutely! I can definitely help you with your crossword puzzle.  To give you the best assistance, I need a bit more information:* **Can you provide the clues, especially the ones you're stuck on?**  I'm good at understanding wordplay, synonyms, and different clue types.* **What's the crossword's theme (if any)?**  Themes often give hints about interconnected clues.* **Are there any intersecting letters you already have filled in?**  Those can be crucial for narrowing down possibilities. Once you provide these, I'll do my best to help you crack the clues and solve the puzzle! \\
\midrule
GPT-4-Turbor(1106) output: & Absolutely! I'd be delighted to help you with your crossword puzzle. Just tell me the clues and the number of letters needed for each answer (and any letters you may have already filled in), and I'll do my best to help you find the solutions. If you have any specific questions about how to approach solving crossword puzzles in general, feel free to ask those as well! \\
\bottomrule
\end{tabular}
\caption{Output example from Gemma base model.}
\label{tab:output_example}
\end{table*}

%% file: table/notation.tex
\begin{table*}[!ht]
    \centering
    \resizebox{\textwidth}{!}{
    \small
    \begin{tabular}{lcl}
    \toprule
    \textbf{Name} & \textbf{Notation} & \multicolumn{1}{c}{\textbf{Description}} \\ \midrule
    Prompt & $x$ & Prompt is the input sequence passed to the model. \\ 
    Response & $y$ & Response is the output sequence of the model. \\ 
    Preferred Response & $y_w$ & The sample has a higher reward score or is preferred by humans. \\ 
    Non-preferred Response & $y_l$ & The sample has a lower reward score or non-prefered by humans. \\ 
    \midrule
    Reward of Preferred Response & $r_w$ & Reward score of $y_w$. \\  
    Reward of Less Preferred Response & $r_l$ & Reward score of $y_l$. \\  
    Reward Margin & $m$ & Reward margin between $r_w$ and $r_l$, \ie $m = r_w - r_l$. \\  
    Margin Batchsize & $m_n$ & The number of preference pairs in one batch. \\  
    Initial Target Margin & $\gamma_0$ & The initial target margin for preference learning. \\
    Adaptive Target Margin & $\gamma_i$ & The adaptive target margin for preference learning. \\ 
    Uniform-distribution & $p_0$ & Uniform distribution of $\gamma_0$: $p_i = 1 / m_n$. \\ 
    $\gamma$-distribution & $p_i$ & The normalized distribution of $\gamma_i$: $p_i = \gamma_i^* / \sum_j \gamma_j^*, \; \text{where} \; \gamma_i^* = \max(0, \gamma_i)$. \\ 
    \midrule
    Policy Model & $\pi_\theta$ & The generative model accepts prompts and outputs response. \\ 
    Reference Model & $\pi_\text{ref}$ & The generative model is employed as a reference to ensure \\
    & & minimal deviation from the policy model. \\ 
    \midrule
    Preference Dataset & $\mathcal{D}$ & Dataset comprising a set of triples $(x,y_w,y_l)$. \\
    Loss Function & $\mathcal{L}$ & Loss function. \\
    Historical Reward Gap Queue & $\mathcal{H}$ & The First-In-First-Out queue to maintain the recent reward gaps. \\
    \midrule
    Smoothing Parameter & $\varepsilon$ & The initial margin for preference learning, also considered as the \\
    & & label flip rate. \\
    Dynamic Smoothing Parameter & $\varepsilon_i$ & The smoothing parameter that adaptive to different samples. \\ 
    Regularization Hyper-parameters & $\tau$ & Regularization Hyper-parameters control the deviation from \\
    & & initial margin. \\
    \bottomrule
    \end{tabular}
}
\caption{Table of Terminology and Notation.}
\label{tab:notations}
\end{table*}

%% file: table/mt_bench.tex
\setlength{\tabcolsep}{3pt}
\begin{table*}[!t]
\small
\centering
\resizebox{1.0\textwidth}{!}{
\begin{tabular}{l|cccc|c}
\toprule
\textbf{Method} & \textbf{Llama3-Instruct (8B)} & \textbf{Mistral-Instruct (7B)} & \textbf{Gemma2-Instruct (9B)} & \textbf{Qwen2.5-Instruct (7B)} & Avg. (rank)\\
\midrule
SFT & 8.1 & 7.5 & 8.4 & 8.5 & 8.13(11) \\
IPO & 9.1 & 8.4 & 9.2 & 8.8 & 8.88(8)\phantom{k} \\
KTO & 9.1 & 7.9 & 9.2 & 8.9 & 8.78(10) \\
CPO & 8.9 & 8.8 & 8.8 & 8.9 & 8.85(9)\phantom{k} \\
ORPO & 8.8 & \underline{\textbf{9.1}} & 8.9 & 8.9 & 8.93(7)\phantom{k}\\
R-DPO & 9.0 & 8.8 & 9.2 & \underline{9.0} & 9.00(5)\phantom{k}\\
DPO & 9.1 & 8.9 & 9.2 & \underline{9.0} & \underline{9.05}(3)\phantom{k}\\
$\beta$-DPO & \underline{9.2} & 8.7 & \underline{\textbf{9.3}} & \underline{9.0} & \underline{9.05}(3)\phantom{k} \\
SimPO & 9.0 & 8.9 & 9.0 & \underline{9.0} & 8.98(6)\phantom{k}\\
\midrule
$\gamma$-DPO & 9.2 & 8.8 & \textbf{9.3} & \textbf{9.1} & 9.10(2)\phantom{k}\\
$\gamma$-SimPO & \textbf{9.3} & \textbf{9.1} & 9.2 & \textbf{9.1} & \textbf{9.18(1)}\phantom{k} \\
\midrule
Improve & \color{+}{+1.1\%} & \color{+}{+0.0\%} & \color{+}{+0.0\%} & \color{+}{+1.1\%} & \color{+}{+1.4\%} \\
\bottomrule
\end{tabular}}
\caption{The Win Rate from MT-Bench~\cite{mt_bench} across different methods. \textbf{Bold}
indicates the best performance for each metric, while \underline{underlined} values represent the best performance excluding our methods, \ie $\gamma$-DPO and $\gamma$-SimPO. ``Improve'' denotes the percentage improvement of the bold value over the underlined one. This metric is considered better when it has a higher value.}
\label{tab:mt_bench}
\end{table*}
\setlength{\tabcolsep}{2pt}

%% file: main.bbl
\begin{thebibliography}{53}
\providecommand{\natexlab}[1]{#1}

\bibitem[{Amini et~al.(2024)Amini, Vieira, and Cotterell}]{amini2024DPOwithoffset}
Afra Amini, Tim Vieira, and Ryan Cotterell. 2024.
\newblock Direct preference optimization with an offset.
\newblock In \emph{Findings of the 62th Annual Meeting of the Association for Computational Linguistics}, pages 9954--9972, Bangkok, Thailand.

\bibitem[{Anil et~al.(2023)Anil, Borgeaud, Wu, Alayrac, Yu, Soricut, Schalkwyk, Dai, Hauth, Millican, Silver, Petrov, Johnson, Antonoglou, Schrittwieser, Glaese, Chen, Pitler, Lillicrap, Lazaridou, Firat, Molloy, Isard, Barham, Hennigan, Lee, Viola, Reynolds, Xu, Doherty, Collins, Meyer, Rutherford, Moreira, Ayoub, Goel, Tucker, Piqueras, Krikun, Barr, Savinov, Danihelka, Roelofs, White, Andreassen, von Glehn, Yagati, Kazemi, Gonzalez, Khalman, Sygnowski, and et~al.}]{team2023gemini}
Rohan Anil, Sebastian Borgeaud, Yonghui Wu, Jean{-}Baptiste Alayrac, Jiahui Yu, Radu Soricut, Johan Schalkwyk, Andrew~M. Dai, Anja Hauth, Katie Millican, David Silver, Slav Petrov, Melvin Johnson, Ioannis Antonoglou, Julian Schrittwieser, Amelia Glaese, Jilin Chen, Emily Pitler, Timothy~P. Lillicrap, Angeliki Lazaridou, Orhan Firat, James Molloy, Michael Isard, Paul~Ronald Barham, Tom Hennigan, Benjamin Lee, Fabio Viola, Malcolm Reynolds, Yuanzhong Xu, Ryan Doherty, Eli Collins, Clemens Meyer, Eliza Rutherford, Erica Moreira, Kareem Ayoub, Megha Goel, George Tucker, Enrique Piqueras, Maxim Krikun, Iain Barr, Nikolay Savinov, Ivo Danihelka, Becca Roelofs, Ana{\"{\i}}s White, Anders Andreassen, Tamara von Glehn, Lakshman Yagati, Mehran Kazemi, Lucas Gonzalez, Misha Khalman, Jakub Sygnowski, and et~al. 2023.
\newblock Gemini: {A} family of highly capable multimodal models.
\newblock \emph{arXiv preprint}, abs/2312.11805.

\bibitem[{Anthony et~al.(2017)Anthony, Tian, and Barber}]{anthony2017thinking-ppo}
Thomas Anthony, Zheng Tian, and David Barber. 2017.
\newblock Thinking fast and slow with deep learning and tree search.
\newblock In \emph{Advances in Neural Information Processing Systems 30}, pages 5360--5370, Long Beach, CA.

\bibitem[{Azar et~al.(2024{\natexlab{a}})Azar, Guo, Piot, Munos, Rowland, Valko, and Calandriello}]{azar2023IPO}
Mohammad~Gheshlaghi Azar, Zhaohan~Daniel Guo, Bilal Piot, R{\'{e}}mi Munos, Mark Rowland, Michal Valko, and Daniele Calandriello. 2024{\natexlab{a}}.
\newblock A general theoretical paradigm to understand learning from human preferences.
\newblock In \emph{Proceedings of the International Conference on Artificial Intelligence and Statistics}, volume 238, pages 4447--4455, Valencia, Spain.

\bibitem[{Azar et~al.(2024{\natexlab{b}})Azar, Guo, Piot, Munos, Rowland, Valko, and Calandriello}]{azar2024general}
Mohammad~Gheshlaghi Azar, Zhaohan~Daniel Guo, Bilal Piot, R{\'{e}}mi Munos, Mark Rowland, Michal Valko, and Daniele Calandriello. 2024{\natexlab{b}}.
\newblock A general theoretical paradigm to understand learning from human preferences.
\newblock In \emph{Proceedings of the International Conference on Artificial Intelligence and Statistics}, volume 238, pages 4447--4455, Palau de Congressos, Valencia.

\bibitem[{Bai et~al.(2022)Bai, Jones, Ndousse, Askell, Chen, DasSarma, Drain, Fort, Ganguli, Henighan, Joseph, Kadavath, Kernion, Conerly, Showk, Elhage, Hatfield{-}Dodds, Hernandez, Hume, Johnston, Kravec, Lovitt, Nanda, Olsson, Amodei, Brown, Clark, McCandlish, Olah, Mann, and Kaplan}]{bai2022training}
Yuntao Bai, Andy Jones, Kamal Ndousse, Amanda Askell, Anna Chen, Nova DasSarma, Dawn Drain, Stanislav Fort, Deep Ganguli, Tom Henighan, Nicholas Joseph, Saurav Kadavath, Jackson Kernion, Tom Conerly, Sheer~El Showk, Nelson Elhage, Zac Hatfield{-}Dodds, Danny Hernandez, Tristan Hume, Scott Johnston, Shauna Kravec, Liane Lovitt, Neel Nanda, Catherine Olsson, Dario Amodei, Tom~B. Brown, Jack Clark, Sam McCandlish, Chris Olah, Benjamin Mann, and Jared Kaplan. 2022.
\newblock Training a helpful and harmless assistant with reinforcement learning from human feedback.
\newblock \emph{arXiv preprint}, abs/2204.05862.

\bibitem[{Bradley and Terry(1952)}]{bradley1952BTmodel}
Ralph~Allan Bradley and Milton~E Terry. 1952.
\newblock Rank analysis of incomplete block designs: I. the method of paired comparisons.
\newblock \emph{Biometrika}, 39(3/4):324--345.

\bibitem[{Chen et~al.(2024)Chen, Zhu, Chen, Soselia, Zhou, Goldstein, Huang, Shoeybi, and Catanzaro}]{chen2024odin-reward}
Lichang Chen, Chen Zhu, Jiuhai Chen, Davit Soselia, Tianyi Zhou, Tom Goldstein, Heng Huang, Mohammad Shoeybi, and Bryan Catanzaro. 2024.
\newblock {ODIN:} disentangled reward mitigates hacking in {RLHF}.
\newblock In \emph{Proceedings of the 41st International Conference on Machine Learning}, Vienna, Austria.

\bibitem[{Chowdhury et~al.(2024)Chowdhury, Kini, and Natarajan}]{chowdhury2024rDPO}
Sayak~Ray Chowdhury, Anush Kini, and Nagarajan Natarajan. 2024.
\newblock Provably robust {DPO:} aligning language models with noisy feedback.
\newblock In \emph{Proceedings of the Forty-first International Conference on Machine Learning}, Vienna, Austria.

\bibitem[{Christiano et~al.(2017)Christiano, Leike, Brown, Martic, Legg, and Amodei}]{christiano2017deep}
Paul~F. Christiano, Jan Leike, Tom~B. Brown, Miljan Martic, Shane Legg, and Dario Amodei. 2017.
\newblock Deep reinforcement learning from human preferences.
\newblock In \emph{Advances in Neural Information Processing Systems 30}, pages 4299--4307, Long Beach, CA.

\bibitem[{Cui et~al.(2024)Cui, Yuan, Ding, Yao, He, Zhu, Ni, Xie, Xie, Lin, Liu, and Sun}]{cui2024ultrafeedback}
Ganqu Cui, Lifan Yuan, Ning Ding, Guanming Yao, Bingxiang He, Wei Zhu, Yuan Ni, Guotong Xie, Ruobing Xie, Yankai Lin, Zhiyuan Liu, and Maosong Sun. 2024.
\newblock {ULTRAFEEDBACK:} boosting language models with scaled {AI} feedback.
\newblock In \emph{Proceedings of the 41st International Conference on Machine Learning}, Vienna, Austria.

\bibitem[{Ding et~al.(2023)Ding, Chen, Xu, Qin, Hu, Liu, Sun, and Zhou}]{ding2023enhancing-sft}
Ning Ding, Yulin Chen, Bokai Xu, Yujia Qin, Shengding Hu, Zhiyuan Liu, Maosong Sun, and Bowen Zhou. 2023.
\newblock Enhancing chat language models by scaling high-quality instructional conversations.
\newblock In \emph{Proceedings of the 2023 Conference on Empirical Methods in Natural Language Processing}, pages 3029--3051, Singapore.

\bibitem[{Dong et~al.(2024)Dong, Xiong, Pang, Wang, Zhao, Zhou, Jiang, Sahoo, Xiong, and Zhang}]{dong2024rlhfworkflowrewardmodeling}
Hanze Dong, Wei Xiong, Bo~Pang, Haoxiang Wang, Han Zhao, Yingbo Zhou, Nan Jiang, Doyen Sahoo, Caiming Xiong, and Tong Zhang. 2024.
\newblock {RLHF} workflow: From reward modeling to online {RLHF}.
\newblock \emph{arXiv preprint}, abs/2405.07863.

\bibitem[{Dubey et~al.(2024)Dubey, Jauhri, Pandey, Kadian, Al{-}Dahle, Letman, Mathur, Schelten, Yang, Fan, Goyal, Hartshorn, Yang, Mitra, Sravankumar, Korenev, Hinsvark, Rao, Zhang, Rodriguez, Gregerson, Spataru, Rozi{\`{e}}re, Biron, Tang, Chern, Caucheteux, Nayak, Bi, Marra, McConnell, Keller, Touret, Wu, Wong, Ferrer, Nikolaidis, Allonsius, Song, Pintz, Livshits, Esiobu, Choudhary, Mahajan, Garcia{-}Olano, Perino, Hupkes, Lakomkin, AlBadawy, Lobanova, Dinan, Smith, Radenovic, Zhang, Synnaeve, Lee, Anderson, Nail, Mialon, Pang, Cucurell, Nguyen, Korevaar, Xu, Touvron, Zarov, Ibarra, Kloumann, Misra, Evtimov, Copet, Lee, Geffert, Vranes, Park, Mahadeokar, Shah, van~der Linde, Billock, Hong, Lee, Fu, Chi, Huang, Liu, Wang, Yu, Bitton, Spisak, Park, Rocca, Johnstun, Saxe, Jia, Alwala, Upasani, Plawiak, Li, Heafield, Stone, and et~al.}]{grattafiori2024llama3}
Abhimanyu Dubey, Abhinav Jauhri, Abhinav Pandey, Abhishek Kadian, Ahmad Al{-}Dahle, Aiesha Letman, Akhil Mathur, Alan Schelten, Amy Yang, Angela Fan, Anirudh Goyal, Anthony Hartshorn, Aobo Yang, Archi Mitra, Archie Sravankumar, Artem Korenev, Arthur Hinsvark, Arun Rao, Aston Zhang, Aur{\'{e}}lien Rodriguez, Austen Gregerson, Ava Spataru, Baptiste Rozi{\`{e}}re, Bethany Biron, Binh Tang, Bobbie Chern, Charlotte Caucheteux, Chaya Nayak, Chloe Bi, Chris Marra, Chris McConnell, Christian Keller, Christophe Touret, Chunyang Wu, Corinne Wong, Cristian~Canton Ferrer, Cyrus Nikolaidis, Damien Allonsius, Daniel Song, Danielle Pintz, Danny Livshits, David Esiobu, Dhruv Choudhary, Dhruv Mahajan, Diego Garcia{-}Olano, Diego Perino, Dieuwke Hupkes, Egor Lakomkin, Ehab AlBadawy, Elina Lobanova, Emily Dinan, Eric~Michael Smith, Filip Radenovic, Frank Zhang, Gabriel Synnaeve, Gabrielle Lee, Georgia~Lewis Anderson, Graeme Nail, Gr{\'{e}}goire Mialon, Guan Pang, Guillem Cucurell, Hailey Nguyen, Hannah Korevaar, Hu~Xu, Hugo
  Touvron, Iliyan Zarov, Imanol~Arrieta Ibarra, Isabel~M. Kloumann, Ishan Misra, Ivan Evtimov, Jade Copet, Jaewon Lee, Jan Geffert, Jana Vranes, Jason Park, Jay Mahadeokar, Jeet Shah, Jelmer van~der Linde, Jennifer Billock, Jenny Hong, Jenya Lee, Jeremy Fu, Jianfeng Chi, Jianyu Huang, Jiawen Liu, Jie Wang, Jiecao Yu, Joanna Bitton, Joe Spisak, Jongsoo Park, Joseph Rocca, Joshua Johnstun, Joshua Saxe, Junteng Jia, Kalyan~Vasuden Alwala, Kartikeya Upasani, Kate Plawiak, Ke~Li, Kenneth Heafield, Kevin Stone, and et~al. 2024.
\newblock The llama 3 herd of models.
\newblock \emph{arXiv preprint}, abs/2407.21783.

\bibitem[{Dubois et~al.(2024)Dubois, Galambosi, Liang, and Hashimoto}]{dubois2024LCalpacaeval}
Yann Dubois, Bal{\'{a}}zs Galambosi, Percy Liang, and Tatsunori~B. Hashimoto. 2024.
\newblock Length-controlled alpacaeval: {A} simple way to debias automatic evaluators.
\newblock \emph{arXiv preprint}, abs/2404.04475.

\bibitem[{Ethayarajh et~al.(2024)Ethayarajh, Xu, Muennighoff, Jurafsky, and Kiela}]{ethayarajh2024KTO}
Kawin Ethayarajh, Winnie Xu, Niklas Muennighoff, Dan Jurafsky, and Douwe Kiela. 2024.
\newblock {KTO:} model alignment as prospect theoretic optimization.

\bibitem[{Gao et~al.(2024)Gao, Schulman, and Hilton}]{gao2023scaling-reward}
Leo Gao, John Schulman, and Jacob Hilton. 2024.
\newblock Scaling laws for reward model overoptimization.
\newblock In \emph{Proceedings of the 41st International Conference on Machine Learning}, Vienna, Austria.

\bibitem[{Havrilla et~al.(2024)Havrilla, Raparthy, Nalmpantis, Dwivedi{-}Yu, Zhuravinskyi, Hambro, and Raileanu}]{havrilla2024glore-reward}
Alexander Havrilla, Sharath~Chandra Raparthy, Christoforos Nalmpantis, Jane Dwivedi{-}Yu, Maksym Zhuravinskyi, Eric Hambro, and Roberta Raileanu. 2024.
\newblock Glore: When, where, and how to improve {LLM} reasoning via global and local refinements.
\newblock In \emph{Proceedings of the 41st International Conference on Machine Learning}, Vienna, Austria.

\bibitem[{Hong et~al.(2024)Hong, Lee, and Thorne}]{hong2024ORPO}
Jiwoo Hong, Noah Lee, and James Thorne. 2024.
\newblock {ORPO:} monolithic preference optimization without reference model.
\newblock In \emph{Proceedings of the 2024 Conference on Empirical Methods in Natural Language Processing}, pages 11170--11189, Miami, FL.

\bibitem[{Huang et~al.(2025)Huang, Wu, Chen, Wang, Gao, Ding, Wu, He, and Wang}]{reward_gap}
Kexin Huang, Junkang Wu, Ziqian Chen, Xue Wang, Jinyang Gao, Bolin Ding, Jiancan Wu, Xiangnan He, and Xiang Wang. 2025.
\newblock Larger or smaller reward margins to select preferences for alignment?
\newblock In \emph{Proceedings of the 42nd International Conference on Machine Learning}, Vancouver, Canada.

\bibitem[{Ji et~al.(2023)Ji, Liu, Dai, Pan, Zhang, Bian, Chen, Sun, Wang, and Yang}]{ji2024beavertails}
Jiaming Ji, Mickel Liu, Josef Dai, Xuehai Pan, Chi Zhang, Ce~Bian, Boyuan Chen, Ruiyang Sun, Yizhou Wang, and Yaodong Yang. 2023.
\newblock Beavertails: Towards improved safety alignment of {LLM} via a human-preference dataset.
\newblock In \emph{Advances in Neural Information Processing Systems 36}, New Orleans, LA.

\bibitem[{Jiang et~al.(2023)Jiang, Sablayrolles, Mensch, Bamford, Chaplot, de~Las~Casas, Bressand, Lengyel, Lample, Saulnier, Lavaud, Lachaux, Stock, Scao, Lavril, Wang, Lacroix, and Sayed}]{jiang2023mistral7b}
Albert~Q. Jiang, Alexandre Sablayrolles, Arthur Mensch, Chris Bamford, Devendra~Singh Chaplot, Diego de~Las~Casas, Florian Bressand, Gianna Lengyel, Guillaume Lample, Lucile Saulnier, L{\'{e}}lio~Renard Lavaud, Marie{-}Anne Lachaux, Pierre Stock, Teven~Le Scao, Thibaut Lavril, Thomas Wang, Timoth{\'{e}}e Lacroix, and William~El Sayed. 2023.
\newblock Mistral 7b.
\newblock \emph{arXiv preprint}, abs/2310.06825.

\bibitem[{Kingma and Ba(2015)}]{kingma2017adam}
Diederik~P. Kingma and Jimmy Ba. 2015.
\newblock Adam: {A} method for stochastic optimization.
\newblock In \emph{Proceedings of the 3rd International Conference on Learning Representations}, San Diego, CA.

\bibitem[{K{\"{o}}pf et~al.(2023)K{\"{o}}pf, Kilcher, von R{\"{u}}tte, Anagnostidis, Tam, Stevens, Barhoum, Nguyen, Stanley, Nagyfi, ES, Suri, Glushkov, Dantuluri, Maguire, Schuhmann, Nguyen, and Mattick}]{kopf2024openassistant-sft}
Andreas K{\"{o}}pf, Yannic Kilcher, Dimitri von R{\"{u}}tte, Sotiris Anagnostidis, Zhi~Rui Tam, Keith Stevens, Abdullah Barhoum, Duc Nguyen, Oliver Stanley, Rich{\'{a}}rd Nagyfi, Shahul ES, Sameer Suri, David Glushkov, Arnav Dantuluri, Andrew Maguire, Christoph Schuhmann, Huu Nguyen, and Alexander Mattick. 2023.
\newblock Openassistant conversations - democratizing large language model alignment.
\newblock In \emph{Advances in Neural Information Processing Systems 36}, New Orleans, LA.

\bibitem[{Kullback and Leibler(1951)}]{kullback1951KLinformation}
Solomon Kullback and Richard~A Leibler. 1951.
\newblock On information and sufficiency.
\newblock \emph{The annals of mathematical statistics}, 22(1):79--86.

\bibitem[{Lambert et~al.(2024)Lambert, Pyatkin, Morrison, Miranda, Lin, Chandu, Dziri, Kumar, Zick, Choi, Smith, and Hajishirzi}]{lambert2024rewardbench-reward}
Nathan Lambert, Valentina Pyatkin, Jacob Morrison, LJ~Miranda, Bill~Yuchen Lin, Khyathi~Raghavi Chandu, Nouha Dziri, Sachin Kumar, Tom Zick, Yejin Choi, Noah~A. Smith, and Hannaneh Hajishirzi. 2024.
\newblock Rewardbench: Evaluating reward models for language modeling.
\newblock \emph{arXiv preprint}, abs/2403.13787.

\bibitem[{Li et~al.(2024)Li, Chiang, Evan~Frick, Zhu, Gonzalez, and Stoica}]{arenahard2024_ArenaHard}
Tianle Li, Wei-Lin Chiang, Lisa~Dunlap Evan~Frick, Banghua Zhu, Joseph~E. Gonzalez, and Ion Stoica. 2024.
\newblock From live data to high-quality benchmarks: The arena-hard pipeline.
\newblock \emph{arXiv preprint}, abs/2406.11939.

\bibitem[{Li and Zhang(2023)}]{li2023alpacaeval}
Xuechen Li and Tianyi Zhang. 2023.
\newblock Alpacaeval: An automatic evaluator of instruction-following models.

\bibitem[{Liang et~al.(2020)Liang, Zou, and Yu}]{liang2020Likert}
Weixin Liang, James Zou, and Zhou Yu. 2020.
\newblock Beyond user self-reported likert scale ratings: {A} comparison model for automatic dialog evaluation.
\newblock In \emph{Proceedings of the 58th Annual Meeting of the Association for Computational Linguistics}, pages 1363--1374, Virtual Event.

\bibitem[{Lightman et~al.(2024)Lightman, Kosaraju, Burda, Edwards, Baker, Lee, Leike, Schulman, Sutskever, and Cobbe}]{lightman2023let-reward}
Hunter Lightman, Vineet Kosaraju, Yuri Burda, Harrison Edwards, Bowen Baker, Teddy Lee, Jan Leike, John Schulman, Ilya Sutskever, and Karl Cobbe. 2024.
\newblock Let's verify step by step.
\newblock In \emph{Proceedings of the 12th International Conference on Learning Representations}, Vienna, Austria.

\bibitem[{Luo et~al.(2023)Luo, Sun, Xu, Zhao, Lou, Tao, Geng, Lin, Chen, and Zhang}]{luo2023wizardmath-reward}
Haipeng Luo, Qingfeng Sun, Can Xu, Pu~Zhao, Jianguang Lou, Chongyang Tao, Xiubo Geng, Qingwei Lin, Shifeng Chen, and Dongmei Zhang. 2023.
\newblock Wizardmath: Empowering mathematical reasoning for large language models via reinforced evol-instruct.
\newblock \emph{arXiv preprint}, abs/2308.09583.

\bibitem[{Meng et~al.(2024)Meng, Xia, and Chen}]{meng2024simpo}
Yu~Meng, Mengzhou Xia, and Danqi Chen. 2024.
\newblock Simpo: Simple preference optimization with a reference-free reward.
\newblock In \emph{Advances in Neural Information Processing Systems 38}, Vancouver, Canada.

\bibitem[{OpenAI(2023)}]{openai2024gpt4technicalreport}
OpenAI. 2023.
\newblock {GPT-4} technical report.
\newblock \emph{arXiv preprint}, abs/2303.08774.

\bibitem[{Ouyang et~al.(2022{\natexlab{a}})Ouyang, Wu, Jiang, Almeida, Wainwright, Mishkin, Zhang, Agarwal, Slama, Ray, Schulman, Hilton, Kelton, Miller, Simens, Askell, Welinder, Christiano, Leike, and Lowe}]{ouyang2022RLHF}
Long Ouyang, Jeffrey Wu, Xu~Jiang, Diogo Almeida, Carroll~L. Wainwright, Pamela Mishkin, Chong Zhang, Sandhini Agarwal, Katarina Slama, Alex Ray, John Schulman, Jacob Hilton, Fraser Kelton, Luke Miller, Maddie Simens, Amanda Askell, Peter Welinder, Paul~F. Christiano, Jan Leike, and Ryan Lowe. 2022{\natexlab{a}}.
\newblock Training language models to follow instructions with human feedback.
\newblock In \emph{Advances in Neural Information Processing Systems 35}, New Orleans, LA.

\bibitem[{Ouyang et~al.(2022{\natexlab{b}})Ouyang, Wu, Jiang, Almeida, Wainwright, Mishkin, Zhang, Agarwal, Slama, Ray, Schulman, Hilton, Kelton, Miller, Simens, Askell, Welinder, Christiano, Leike, and Lowe}]{ouyang2022training}
Long Ouyang, Jeffrey Wu, Xu~Jiang, Diogo Almeida, Carroll~L. Wainwright, Pamela Mishkin, Chong Zhang, Sandhini Agarwal, Katarina Slama, Alex Ray, John Schulman, Jacob Hilton, Fraser Kelton, Luke Miller, Maddie Simens, Amanda Askell, Peter Welinder, Paul~F. Christiano, Jan Leike, and Ryan Lowe. 2022{\natexlab{b}}.
\newblock Training language models to follow instructions with human feedback.
\newblock In \emph{Advances in Neural Information Processing Systems 35}, New Orleans, LA.

\bibitem[{Park et~al.(2024)Park, Rafailov, Ermon, and Finn}]{park2024R-DPO}
Ryan Park, Rafael Rafailov, Stefano Ermon, and Chelsea Finn. 2024.
\newblock Disentangling length from quality in direct preference optimization.
\newblock In \emph{Findings of the 62th Annual Meeting of the Association for Computational Linguistics}, pages 4998--5017, Bangkok, Thailand.

\bibitem[{Rafailov et~al.(2023)Rafailov, Sharma, Mitchell, Manning, Ermon, and Finn}]{rafailov2024DPO}
Rafael Rafailov, Archit Sharma, Eric Mitchell, Christopher~D. Manning, Stefano Ermon, and Chelsea Finn. 2023.
\newblock Direct preference optimization: Your language model is secretly a reward model.
\newblock In \emph{Advances in Neural Information Processing Systems 36}, New Orleans, LA.

\bibitem[{Rivi{\`{e}}re et~al.(2024)Rivi{\`{e}}re, Pathak, Sessa, Hardin, Bhupatiraju, Hussenot, Mesnard, Shahriari, Ram{\'{e}}, Ferret, Liu, Tafti, Friesen, Casbon, Ramos, Kumar, Lan, Jerome, Tsitsulin, Vieillard, Stanczyk, Girgin, Momchev, Hoffman, Thakoor, Grill, Neyshabur, Bachem, Walton, Severyn, Parrish, Ahmad, Hutchison, Abdagic, Carl, Shen, Brock, Coenen, Laforge, Paterson, Bastian, Piot, Wu, Royal, Chen, Kumar, Perry, Welty, Choquette{-}Choo, Sinopalnikov, Weinberger, Vijaykumar, Rogozinska, Herbison, Bandy, Wang, Noland, Moreira, Senter, Eltyshev, Visin, Rasskin, Wei, Cameron, Martins, Hashemi, Klimczak{-}Plucinska, Batra, Dhand, Nardini, Mein, Zhou, Svensson, Stanway, Chan, Zhou, Carrasqueira, Iljazi, Becker, Fernandez, van Amersfoort, Gordon, Lipschultz, Newlan, Ji, Mohamed, Badola, Black, Millican, McDonell, Nguyen, Sodhia, Greene, Sj{\"{o}}sund, Usui, Sifre, Heuermann, Lago, and McNealus}]{gemmateam2024gemma2}
Morgane Rivi{\`{e}}re, Shreya Pathak, Pier~Giuseppe Sessa, Cassidy Hardin, Surya Bhupatiraju, L{\'{e}}onard Hussenot, Thomas Mesnard, Bobak Shahriari, Alexandre Ram{\'{e}}, Johan Ferret, Peter Liu, Pouya Tafti, Abe Friesen, Michelle Casbon, Sabela Ramos, Ravin Kumar, Charline~Le Lan, Sammy Jerome, Anton Tsitsulin, Nino Vieillard, Piotr Stanczyk, Sertan Girgin, Nikola Momchev, Matt Hoffman, Shantanu Thakoor, Jean{-}Bastien Grill, Behnam Neyshabur, Olivier Bachem, Alanna Walton, Aliaksei Severyn, Alicia Parrish, Aliya Ahmad, Allen Hutchison, Alvin Abdagic, Amanda Carl, Amy Shen, Andy Brock, Andy Coenen, Anthony Laforge, Antonia Paterson, Ben Bastian, Bilal Piot, Bo~Wu, Brandon Royal, Charlie Chen, Chintu Kumar, Chris Perry, Chris Welty, Christopher~A. Choquette{-}Choo, Danila Sinopalnikov, David Weinberger, Dimple Vijaykumar, Dominika Rogozinska, Dustin Herbison, Elisa Bandy, Emma Wang, Eric Noland, Erica Moreira, Evan Senter, Evgenii Eltyshev, Francesco Visin, Gabriel Rasskin, Gary Wei, Glenn Cameron, Gus
  Martins, Hadi Hashemi, Hanna Klimczak{-}Plucinska, Harleen Batra, Harsh Dhand, Ivan Nardini, Jacinda Mein, Jack Zhou, James Svensson, Jeff Stanway, Jetha Chan, Jin~Peng Zhou, Joana Carrasqueira, Joana Iljazi, Jocelyn Becker, Joe Fernandez, Joost van Amersfoort, Josh Gordon, Josh Lipschultz, Josh Newlan, Ju{-}yeong Ji, Kareem Mohamed, Kartikeya Badola, Kat Black, Katie Millican, Keelin McDonell, Kelvin Nguyen, Kiranbir Sodhia, Kish Greene, Lars~Lowe Sj{\"{o}}sund, Lauren Usui, Laurent Sifre, Lena Heuermann, Leticia Lago, and Lilly McNealus. 2024.
\newblock Gemma 2: Improving open language models at a practical size.
\newblock \emph{arXiv preprint}, abs/2408.00118.

\bibitem[{Schulman et~al.(2017)Schulman, Wolski, Dhariwal, Radford, and Klimov}]{schulman2017PPO}
John Schulman, Filip Wolski, Prafulla Dhariwal, Alec Radford, and Oleg Klimov. 2017.
\newblock Proximal policy optimization algorithms.
\newblock \emph{arXiv preprint}, abs/1707.06347.

\bibitem[{Su et~al.(2022)Su, He, Jiang, Zhang, Zou, and Fan}]{su2022dsladynamicsmoothlabel}
Hu~Su, Yonghao He, Rui Jiang, Jiabin Zhang, Wei Zou, and Bin Fan. 2022.
\newblock {DSLA:} dynamic smooth label assignment for efficient anchor-free object detection.
\newblock \emph{Pattern Recognit}, 131:108868.

\bibitem[{Wang et~al.(2023)Wang, Chen, Pei, Xie, Kang, Zhang, Xu, Xiong, Dutta, Schaeffer, Truong, Arora, Mazeika, Hendrycks, Lin, Cheng, Koyejo, Song, and Li}]{wang2024decodingtrust}
Boxin Wang, Weixin Chen, Hengzhi Pei, Chulin Xie, Mintong Kang, Chenhui Zhang, Chejian Xu, Zidi Xiong, Ritik Dutta, Rylan Schaeffer, Sang~T. Truong, Simran Arora, Mantas Mazeika, Dan Hendrycks, Zinan Lin, Yu~Cheng, Sanmi Koyejo, Dawn Song, and Bo~Li. 2023.
\newblock Decodingtrust: {A} comprehensive assessment of trustworthiness in{GPT} models.
\newblock In \emph{Advances in Neural Information Processing Systems 36}, New Orleans, LA.

\bibitem[{Wu et~al.(2025{\natexlab{a}})Wu, Wang, Yang, Wu, Gao, Ding, Wang, and He}]{wu2024alphadpo}
Junkang Wu, Xue Wang, Zhengyi Yang, Jiancan Wu, Jinyang Gao, Bolin Ding, Xiang Wang, and Xiangnan He. 2025{\natexlab{a}}.
\newblock {\(\alpha\)}-dpo: Adaptive reward margin is what direct preference optimization needs.
\newblock In \emph{Proceedings of the 13rd International Conference on Learning Representations}, Singapore.

\bibitem[{Wu et~al.(2025{\natexlab{b}})Wu, Xie, Yang, Wu, Chen, Gao, Ding, Wang, and He}]{drdpo}
Junkang Wu, Yuexiang Xie, Zhengyi Yang, Jiancan Wu, Jiawei Chen, Jinyang Gao, Bolin Ding, Xiang Wang, and Xiangnan He. 2025{\natexlab{b}}.
\newblock Towards robust alignment of language models: Distributionally robustifying direct preference optimization.
\newblock In \emph{Proceedings of the 13th International Conference on Learning Representations}, Singapore.

\bibitem[{Wu et~al.(2024)Wu, Xie, Yang, Wu, Gao, Ding, Wang, and He}]{wu2024beta_DPO}
Junkang Wu, Yuexiang Xie, Zhengyi Yang, Jiancan Wu, Jinyang Gao, Bolin Ding, Xiang Wang, and Xiangnan He. 2024.
\newblock {\(\beta\)}-dpo: Direct preference optimization with dynamic {\(\beta\)}.
\newblock In \emph{Advances in Neural Information Processing Systems 37}, Vancouver, Canada.

\bibitem[{Xu et~al.(2024)Xu, Sharaf, Chen, Tan, Shen, Durme, Murray, and Kim}]{xu2024CPO}
Haoran Xu, Amr Sharaf, Yunmo Chen, Weiting Tan, Lingfeng Shen, Benjamin~Van Durme, Kenton Murray, and Young~Jin Kim. 2024.
\newblock Contrastive preference optimization: Pushing the boundaries of {LLM} performance in machine translation.
\newblock In \emph{Proceedings of the 41st International Conference on Machine Learning}, Vienna, Austria.

\bibitem[{Yang et~al.(2024)Yang, Yang, Zhang, Hui, Zheng, Yu, Li, Liu, Huang, Wei, Lin, Yang, Tu, Zhang, Yang, Yang, Zhou, Lin, Dang, Lu, Bao, Yang, Yu, Li, Xue, Zhang, Zhu, Men, Lin, Li, Xia, Ren, Ren, Fan, Su, Zhang, Wan, Liu, Cui, Zhang, and Qiu}]{qwen2025qwen25}
An~Yang, Baosong Yang, Beichen Zhang, Binyuan Hui, Bo~Zheng, Bowen Yu, Chengyuan Li, Dayiheng Liu, Fei Huang, Haoran Wei, Huan Lin, Jian Yang, Jianhong Tu, Jianwei Zhang, Jianxin Yang, Jiaxi Yang, Jingren Zhou, Junyang Lin, Kai Dang, Keming Lu, Keqin Bao, Kexin Yang, Le~Yu, Mei Li, Mingfeng Xue, Pei Zhang, Qin Zhu, Rui Men, Runji Lin, Tianhao Li, Tingyu Xia, Xingzhang Ren, Xuancheng Ren, Yang Fan, Yang Su, Yichang Zhang, Yu~Wan, Yuqiong Liu, Zeyu Cui, Zhenru Zhang, and Zihan Qiu. 2024.
\newblock Qwen2.5 technical report.
\newblock \emph{arXiv preprint}, abs/2412.15115.

\bibitem[{Zhao et~al.(2024{\natexlab{a}})Zhao, Winata, Das, Zhang, Yao, Tang, and Sahu}]{zhao2024rainbowpo}
Hanyang Zhao, Genta~Indra Winata, Anirban Das, Shi{-}Xiong Zhang, David~D. Yao, Wenpin Tang, and Sambit Sahu. 2024{\natexlab{a}}.
\newblock Rainbowpo: {A} unified framework for combining improvements in preference optimization.
\newblock \emph{arXiv preprint}, abs/2410.04203.

\bibitem[{Zhao et~al.(2024{\natexlab{b}})Zhao, Ren, Hessel, Cardie, Choi, and Deng}]{zhao2024wildchat}
Wenting Zhao, Xiang Ren, Jack Hessel, Claire Cardie, Yejin Choi, and Yuntian Deng. 2024{\natexlab{b}}.
\newblock Wildchat: 1m chatgpt interaction logs in the wild.
\newblock In \emph{Proceedings of the 12th International Conference on Learning Representations}, Vienna, Austria.

\bibitem[{Zheng et~al.(2023{\natexlab{a}})Zheng, Chiang, Sheng, Zhuang, Wu, Zhuang, Lin, Li, Li, Xing, Zhang, Gonzalez, and Stoica}]{mt_bench}
Lianmin Zheng, Wei{-}Lin Chiang, Ying Sheng, Siyuan Zhuang, Zhanghao Wu, Yonghao Zhuang, Zi~Lin, Zhuohan Li, Dacheng Li, Eric~P. Xing, Hao Zhang, Joseph~E. Gonzalez, and Ion Stoica. 2023{\natexlab{a}}.
\newblock Judging llm-as-a-judge with mt-bench and chatbot arena.
\newblock In \emph{Advances in Neural Information Processing Systems 36}, New Orleans, LA.

\bibitem[{Zheng et~al.(2023{\natexlab{b}})Zheng, Dou, Gao, Hua, Shen, Wang, Liu, Jin, Liu, Zhou, Xiong, Chen, Xi, Xu, Lai, Zhu, Chang, Yin, Weng, Cheng, Huang, Sun, Yan, Gui, Zhang, Qiu, and Huang}]{zheng2023secretsrlhflargelanguage}
Rui Zheng, Shihan Dou, Songyang Gao, Yuan Hua, Wei Shen, Binghai Wang, Yan Liu, Senjie Jin, Qin Liu, Yuhao Zhou, Limao Xiong, Lu~Chen, Zhiheng Xi, Nuo Xu, Wenbin Lai, Minghao Zhu, Cheng Chang, Zhangyue Yin, Rongxiang Weng, Wensen Cheng, Haoran Huang, Tianxiang Sun, Hang Yan, Tao Gui, Qi~Zhang, Xipeng Qiu, and Xuanjing Huang. 2023{\natexlab{b}}.
\newblock Secrets of {RLHF} in large language models part {I:} {PPO}.
\newblock \emph{arXiv preprint}, abs/2307.04964.

\bibitem[{Zhou et~al.(2023)Zhou, Liu, Xu, Iyer, Sun, Mao, Ma, Efrat, Yu, Yu, Zhang, Ghosh, Lewis, Zettlemoyer, and Levy}]{zhou2024lima-sft}
Chunting Zhou, Pengfei Liu, Puxin Xu, Srinivasan Iyer, Jiao Sun, Yuning Mao, Xuezhe Ma, Avia Efrat, Ping Yu, Lili Yu, Susan Zhang, Gargi Ghosh, Mike Lewis, Luke Zettlemoyer, and Omer Levy. 2023.
\newblock {LIMA:} less is more for alignment.
\newblock In \emph{Advances in Neural Information Processing Systems 36}, New Orleans, LA.

\bibitem[{Zhou et~al.(2024)Zhou, Agrawal, Zhang, Indurthi, Zhao, Song, Xu, and Zhu}]{zhou2024wpo}
Wenxuan Zhou, Ravi Agrawal, Shujian Zhang, Sathish~Reddy Indurthi, Sanqiang Zhao, Kaiqiang Song, Silei Xu, and Chenguang Zhu. 2024.
\newblock {WPO:} enhancing {RLHF} with weighted preference optimization.
\newblock In \emph{Proceedings of the 2024 Conference on Empirical Methods in Natural Language Processing}, pages 11170--11189, Miami, FL.

\bibitem[{Ziegler et~al.(2019)Ziegler, Stiennon, Wu, Brown, Radford, Amodei, Christiano, and Irving}]{ziegler2019fine}
Daniel~M. Ziegler, Nisan Stiennon, Jeffrey Wu, Tom~B. Brown, Alec Radford, Dario Amodei, Paul~F. Christiano, and Geoffrey Irving. 2019.
\newblock Fine-tuning language models from human preferences.
\newblock \emph{arXiv preprint}, abs/1909.08593.

\end{thebibliography}
